\documentclass[]{article}
\usepackage{proceed2e-arxiv}
\pdfoutput=1

\usepackage{times}

\usepackage{amsmath,amssymb,amsthm}
\usepackage{graphicx}
\usepackage{caption}
\usepackage{subcaption}

\usepackage{etoolbox}
\newbool{shortversion}
\boolfalse{shortversion}

\usepackage{color}
\usepackage[colorlinks=true,hypertexnames=false]{hyperref}

\usepackage[noabbrev,capitalize]{cleveref}
\crefformat{equation}{(#2#1#3)}
\crefrangeformat{equation}{(#3#1#4) to~(#5#2#6)}
\crefmultiformat{equation}{(#2#1#3)}%
{ and~(#2#1#3)}{, (#2#1#3)}{ and~(#2#1#3)}

\usepackage{autonum}

\usepackage[style=authoryear,backend=bibtex,doi=false,isbn=false,url=false,maxcitenames=2,maxbibnames=5]{biblatex}
\newtheorem{proposition}{Proposition}
\newtheorem*{theorem*}{Theorem}

\newcommand{\acro}[1]{\textsc{\MakeLowercase{#1}}}

\newcommand{\E}{\mathbb{E}}

\newcommand{\M}{\mathcal{M}}

\newcommand{\R}{\mathbb{R}}
\newcommand{\T}{\mathcal{T}}
\newcommand{\X}{\mathcal{X}}

\newcommand{\simiid}{\stackrel{iid}{\sim}}
\newcommand{\tp}{^\mathsf{T}}
\newcommand{\ud}{\mathrm{d}}

\newcommand{\abs}[1]{\lvert #1 \rvert}
\newcommand{\Abs}[1]{\left\lvert #1 \right\rvert}
\newcommand{\norm}[1]{\lVert #1 \rVert}
\newcommand{\Norm}[1]{\left\lVert #1 \right\rVert}

\DeclareMathOperator{\erfc}{erfc}

\DeclareMathOperator{\tr}{tr}
\DeclareMathOperator{\Cov}{Cov}
\DeclareMathOperator{\Var}{Var}

\DeclareMathOperator{\Unif}{Unif}
\DeclareMathOperator*{\argmax}{argmax}
\DeclareMathOperator*{\argmin}{argmin}

\newcommand{\httpurl}[1]{\href{http://#1}{\nolinkurl{#1}}}

\title{On the Error of Random Fourier Features}

\author{ {\bf Dougal J.~Sutherland} \\
Carnegie Mellon University\\
Pittsburgh, PA \\
\texttt{dsutherl@cs.cmu.edu} \\
\And
{\bf Jeff Schneider}   \\
Carnegie Mellon University \\
Pittsburgh, PA \\
\texttt{schneide@cs.cmu.edu} \\
}

\begin{document}

\maketitle


\begin{abstract}
Kernel methods give powerful, flexible, and theoretically grounded approaches to solving many problems in machine learning. The standard approach, however, requires pairwise evaluations of a kernel function, which can lead to scalability issues for very large datasets. \Textcite{rks} suggested a popular approach to handling this problem, known as random Fourier features. The quality of this approximation, however, is not well understood. We improve the uniform error bound of that paper, as well as giving novel understandings of the embedding's variance, approximation error, and use in some machine learning methods. We also point out that surprisingly, of the two main variants of those features, the more widely used is strictly higher-variance for the Gaussian kernel and has worse bounds.
\end{abstract}

\section{INTRODUCTION}

Kernel methods provide an elegant, theoretically well-founded, and powerful
approach to solving many learning problems.
Since traditional algorithms require
the computation of a full $N \times N$ pairwise kernel matrix
to solve learning problems on $N$ input instances, however,
scaling these methods to large-scale datasets containing
more than thousands of data points has proved challenging.
\Textcite{rks} spurred interest in one very attractive approach:
approximating a continuous shift-invariant kernel $k : \X \times \X \to \R$
by
\[
    k(x, y) \approx z(x)\tp z(y) =: s(x, y)
,\]
where $z : \X \to \R^D$.
Then primal methods in $\R^D$ can be used,
allowing most learning problems to be solved in $O(N)$ time \parencite[e.g.][]{joachims2006}.
Recent work has also exploited these embeddings
in some of the most-scalable kernel methods to date \parencite{dai:dsg}.

\Textcite{rks} give two such embeddings,
based on the Fourier transform $P(\omega)$ of the kernel $k$:
one of the form
\[
    \tilde{z}(x) := 
    \sqrt{\frac{2}{D}}
    \begin{bmatrix}
        \sin(\omega_1\tp x)
     \\ \cos(\omega_1\tp x)
     \\ \vdots
     \\ \sin(\omega_{D/2}\tp x)
     \\ \cos(\omega_{D/2}\tp x)
    \end{bmatrix}
    ,\;
    \omega_i \simiid P(\omega)
\label{eq:z-tilde}
\]
and another of the form
\[
    \breve{z}(x) := 
    \sqrt{\frac{2}{D}}
    \begin{bmatrix}
        \cos(\omega_1\tp x + b_1)
     \\ \vdots
     \\ \cos(\omega_D\tp x + b_D)
    \end{bmatrix}
    ,\;
    \begin{aligned}
        \omega_i &\simiid P(\omega)
     \\ b_i &\simiid \Unif_{[0, 2\pi]}
    \end{aligned}
\label{eq:z-breve}
.\]
Bochner's theorem \parencite*{bochner}
guarantees that for any continuous positive-definite function $k(x - y)$,
its Fourier transform will be a nonnegative measure;
if $k(0) = 1$, it will be properly normalized.
Letting $\tilde{s}$ be the reconstruction based on $\tilde{z}$
and $\breve{s}$ that for $\breve{z}$,
we have that:
\begin{align}
      \tilde{s}(x, y)
  & = \frac{1}{D/2} \sum_{i=1}^{D/2} \cos(\omega_i \tp (x - y))
\label{eq:s-tilde}
\\    \breve{s}(x, y)
  & = \frac{1}{D} \sum_{i=1}^D \cos(\omega_i\tp (x - y)) + \cos(\omega_i\tp(x + y) + 2b_i)
\label{eq:s-breve}
.\end{align}
Letting $\Delta := x - y$, we have:
\begin{gather+}
    \E \cos(\omega\tp \Delta)
    = \Re \int\! e^{\omega\tp\Delta \sqrt{-1}} \ud P(\omega)
    = \Re k(\Delta)
\label{eq:emb-mean}
\\
    \E_\omega \E_b \cos(\omega\tp(x + y) + 2 b)
    = 0
\label{eq:emb-noise-mean}
.\end{gather+}
Thus each $s(x, y)$ is a mean of bounded terms with expectation $k(x, y)$.
For a given embedding dimension $D$,
it is not immediately obvious which approximation is preferable:
$\breve{z}$ gives twice as many samples for $\omega$,
but adds additional (non-shift-invariant) noise.
The academic literature seems split on the issue:
of the first 100 papers citing \textcite{rks} in a Google Scholar search,
15 used either $\tilde z$ or the equivalent complex formulation,
14 used $\breve z$,
28 did not specify,
and the remainder didn't use the embedding.
(None discussed that there was a choice.)
Not included in the count are  are Rahimi and Recht's later work \parencite*{rahimi-nips,rahimi-allerton},
which used $\breve z$;
indeed, post-publication revisions of the original paper only discuss $\breve{z}$.
Practically,
we are aware of three implementations in machine learning libraries,
each of which use $\breve z$ at the time of writing:
scikit-learn \parencite{scikit-learn},
Shogun \parencite{shogun},
and JSAT \parencite{jsat}.

We show that $\tilde z$ is superior for the popular Gaussian kernel,
as well as how to decide which to use for other kernels.

The primary previous analyses of these embeddings,
outside the one in the original paper,
have been by \textcite{rahimi-nips},
who bound the increase in error of empirical risk estimates
when learning models in the induced \acro{rkhs},
and by \textcite{yang:nystroem-vs-fourier},
who compare the ability of the Nystr\"om and Fourier embeddings
to exploit eigengaps in the learning problem.
We instead study the approximation directly,
providing a complementary view of the quality of these embeddings.

\Cref{sec:error:variance} studies the variance of each embedding,
showing that which is preferable depends on the kernel as well as the particular value of $\Delta$,
but for the popular Gaussian kernel $\tilde{s}$ is uniformly lower-variance.
\Cref{sec:error:uni} studies uniform convergence bounds,
tightening constants in the original $\tilde z$ bound
and proving a comparable one (with worse constants) for $\breve z$,
bounding the expectation of the maximal error,
and providing exponential concentration about the mean.
\Cref{sec:error:l2} studies the $L_2$ convergence of each approximation;
$\tilde{z}$ is again superior for the Gaussian kernel.
\Cref{sec:downstream} discusses the effect of this approximation error
when used in various machine learning methods.
\Cref{sec:experiments} evaluates the two embeddings and the bounds empirically.

\section{APPROXIMATION ERROR} \label{sec:error}

We will give various analyses of the error due to each approximation.

\subsection{VARIANCE} \label{sec:error:variance}

\Cref{eq:emb-mean,eq:emb-noise-mean} establish that
$\E s(\Delta) = k(\Delta)$.
What about the variance?
We have that
\begin{align}
       \Cov&\left( \tilde{s}(\Delta), \tilde{s}(\Delta') \right)
\\& =  \Cov\left( \frac{2}{D} \sum_{i=1}^{D/2} \cos( \omega_i\tp \Delta ),
                  \frac{2}{D} \sum_{i=1}^{D/2} \cos( \omega_i\tp \Delta' ) \right)
\\& =  \frac{2}{D} \Cov\left( \cos(\omega\tp\Delta), \cos(\omega\tp\Delta') \right)
\\& =  \frac{2}{D} \left[
         \tfrac12 k(\Delta - \Delta')
       + \tfrac12 k(\Delta + \Delta')
       - k(\Delta) k(\Delta')
       \right]
\end{align}
using $\cos(\alpha) \cos(\beta) = \tfrac12 \cos(\alpha + \beta) + \tfrac12 \cos(\alpha - \beta)$
and also
$\E \cos(\omega\tp\Delta) = k(\Delta)$.
Thus
\[
    \Var \tilde s(\Delta) = \frac1D \left[ 1 + k(2 \Delta) - 2 k(\Delta)^2 \right]
    \label{eq:var-s-tilde}
.\]
Similarly, denoting $x + y$ by $t$,
\begin{align}
       \Cov&\left( \breve{s}(x, y), \breve{s}(x', y') \right)
\\& =  \frac{1}{D} \Cov\left(
         \cos(\omega\tp\Delta) + \cos(\omega\tp t + 2b),
\right.\\&\phantom{=\frac{1}{D}\Cov\big(}\left.
         \cos(\omega\tp\Delta') + \cos(\omega\tp t' + 2b)
       \right)
\\& =  \frac{1}{D} \left[
         \tfrac12 k(\Delta - \Delta')
       + \tfrac12 k(\Delta + \Delta')
       - k(\Delta) k(\Delta')
\right.\\&\qquad\left.
       + \tfrac12 k(t - t')
       \right]
\end{align}
which gives
\begin{equation+}
    \Var \breve s(x, y)
    = \frac{1}{D} \left[
        1 + \tfrac12 k(2\Delta) - k(\Delta)^2
    \right]
    \label{eq:var-s-breve}
.\end{equation+}
Thus $\tilde s$ has lower variance than $\breve s$ if
\[
    \Var \cos(\omega\tp\Delta)
    = \frac12 + \frac12 k(2 \Delta) - k(\Delta)^2
    \le \frac12
    \label{eq:var-cos-wd}
.\]
The Gaussian kernel
$k(\Delta) = \exp\left( - \frac{\Norm\Delta^2}{2 \sigma^2} \right)$
has
\[
    \Var \cos(\omega\tp\Delta)
    =
    \frac12
    \left(
        1
      - \exp\left( - \frac{\Norm\Delta^2}{\sigma^2} \right)
    \right)^2
    \le \frac12
,\]
so that $\tilde z$ is always lower-variance than $\breve z$,
and the difference in variance is greatest when $k(\Delta)$ is largest.
This is illustrated in \cref{fig:rbf-var}.

\begin{figure}[bt]
  \centering
  \includegraphics[width=.45\textwidth]{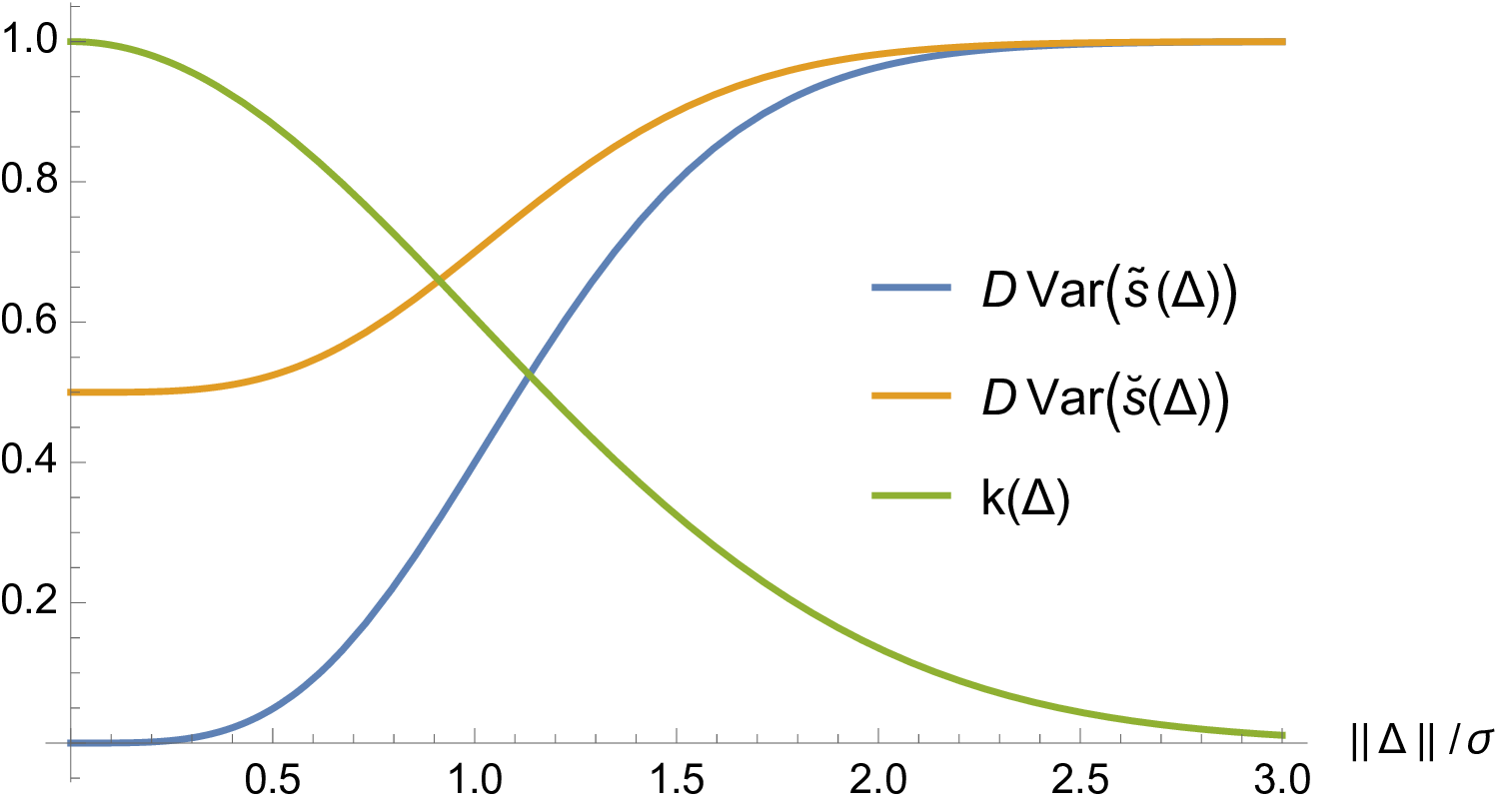}
  \caption{
    The variance per dimension
    of  $\tilde{s}$ (blue)
    and $\breve{s}$ (orange)
    for the Gaussian \acro{rbf} kernel (green).
  }
  \label{fig:rbf-var}
\end{figure}


\subsection{UNIFORM ERROR BOUND} \label{sec:error:uni}

Let $f(x, y) := s(x, y) - k(x, y)$ denote the error of the approximation.
We will investigate $\norm{f}_\infty$,
i.e.\ the maximal approximation error across the domain of $k$.
We first consider the bound given by \textcite{rks},
and then provide a new bound on
$\E \norm{f}_\infty$
and its concentration around that mean.

\subsubsection{Original High-Probability Bound} \label{sec:error:uni:orig}

Claim 1 of \textcite{rks} is that if $\X \subset \R^d$ is compact with diameter $\ell$,\footnote{Note that our $D$ is half of the $D$ in \textcite{rks}, since we want to compare embeddings of the same dimension.}
\[
    \Pr\left( \Norm{f}_\infty \ge \varepsilon \right)
    \le 256 \left( \frac{\sigma_p \ell}{\varepsilon} \right)^2
        \exp\left( - \frac{D \varepsilon^2}{8 (d + 2)} \right)
,\]
where $\sigma_p^2 = \E\left[ \omega\tp\omega \right] = \tr \nabla^2 k(0)$
depends on the kernel.

It is not necessarily clear in that paper that this bound applies only to the $\tilde z$ embedding;
we can also tighten some constants.
We first state the tightened bound for $\tilde z$.

\begin{proposition} \label{thm:uni:tilde}
Let $k$ be a continuous shift-invariant positive-definite function
$k(x, y) = k(\Delta)$ defined on $\X \subset \R^d$,
with $k(0) = 1$ and such that $\nabla^2 k(0)$ exists.
Suppose $\X$ is compact, with diameter $\ell$.
Denote $k$'s Fourier transform as $P(\omega)$, which will be a probability distribution;
let $\sigma_p^2 = \E_p \Norm\omega^2$.
Let $\tilde z$ be as in \cref{eq:z-tilde},
and define $\tilde f(x, y) := \tilde z(x)\tp \tilde z(y) - k(x, y)$.
For any $\varepsilon > 0$, let
\begin{align}
  \alpha_\varepsilon
  &:= \min\left(1, \sup_{x, y \in \X} \frac12 + \frac12 k(2x, 2y) - k(x, y)^2 + \tfrac13 \varepsilon \right),
\\
  \beta_d
  &:= \left( \left( \tfrac{d}{2} \right)^{\frac{-d}{d+2}} + \left( \tfrac{d}{2} \right)^{\frac{2}{d+2}} \right) 2^\frac{6d+2}{d+2}
.\end{align}
Then, assuming only for the second statement that $\varepsilon \le \sigma_p \ell$,
\begin{align}
    \Pr\left( \norm{\tilde f}_\infty \ge \varepsilon \right)
   &\le
    \beta_d 
    \left( \frac{\sigma_p \ell}{\varepsilon} \right)^\frac{2}{1 + \frac{2}{d}}
    \exp\left( - \frac{D \varepsilon^2}{8 (d + 2) \alpha_\varepsilon} \right)
\\ &\le
    66
    \left( \frac{\sigma_p \ell}{\varepsilon} \right)^2
    \exp\left( - \frac{D \varepsilon^2}{8 (d + 2)} \right)
.\end{align}
Thus, we can achieve an embedding with pointwise error no more than $\varepsilon$
with probability at least $1 - \delta$ as long as
\[
    D
    \ge
    \frac{8 (d + 2) \alpha_\varepsilon}{\varepsilon^2}
    \left[
      \frac{2}{1 + \frac{2}{d}}
      \log \frac{\sigma_p \ell}{\varepsilon}
      + \log \frac{\beta_d}{\delta}
    \right]
.\]
\end{proposition}
The proof strategy is very similar to that of \textcite{rks}:
place an $\varepsilon$-net with radius $r$ over $\X_\Delta := \{ x - y : x, y \in \X \}$,
bound the error $\tilde f$ by $\varepsilon / 2$ at the centers of the net by Hoeffding's inequality \parencite*{hoeffding},
and bound the Lipschitz constant of $\tilde f$, which is at most that of $\tilde s$, by $\varepsilon / (2 r)$ with Markov's inequality.
The introduction of $\alpha_\varepsilon$ is by replacing Hoeffding's inequality
with that of \textcite{bernstein} when it is tighter,
using the variance from \cref{eq:var-s-tilde}.
The constant $\beta_d$ is obtained by exactly optimizing the value of $r$,
rather than the algebraically simpler value originally used;
$\beta_{64} = 66$ is its maximum, and $\lim_{d \to \infty} \beta_d = 64$,
though it is lower for small $d$,
as shown in \cref{fig:uni-constants}.
The additional hypothesis, that $\nabla^2 k(0)$ exists,
is equivalent to the existence of the first two moments of $P(\omega)$;
a finite first moment is used in the proof,
and of course without a finite second moment the bound is vacuous.
The full proof is given in
\ifbool{shortversion}{%
  Appendix A.1 (in the supplemental material).
}{%
  \cref{proof:uni:tilde}.
}

\begin{figure}[bt]
  \centering
  \includegraphics[width=.45\textwidth]{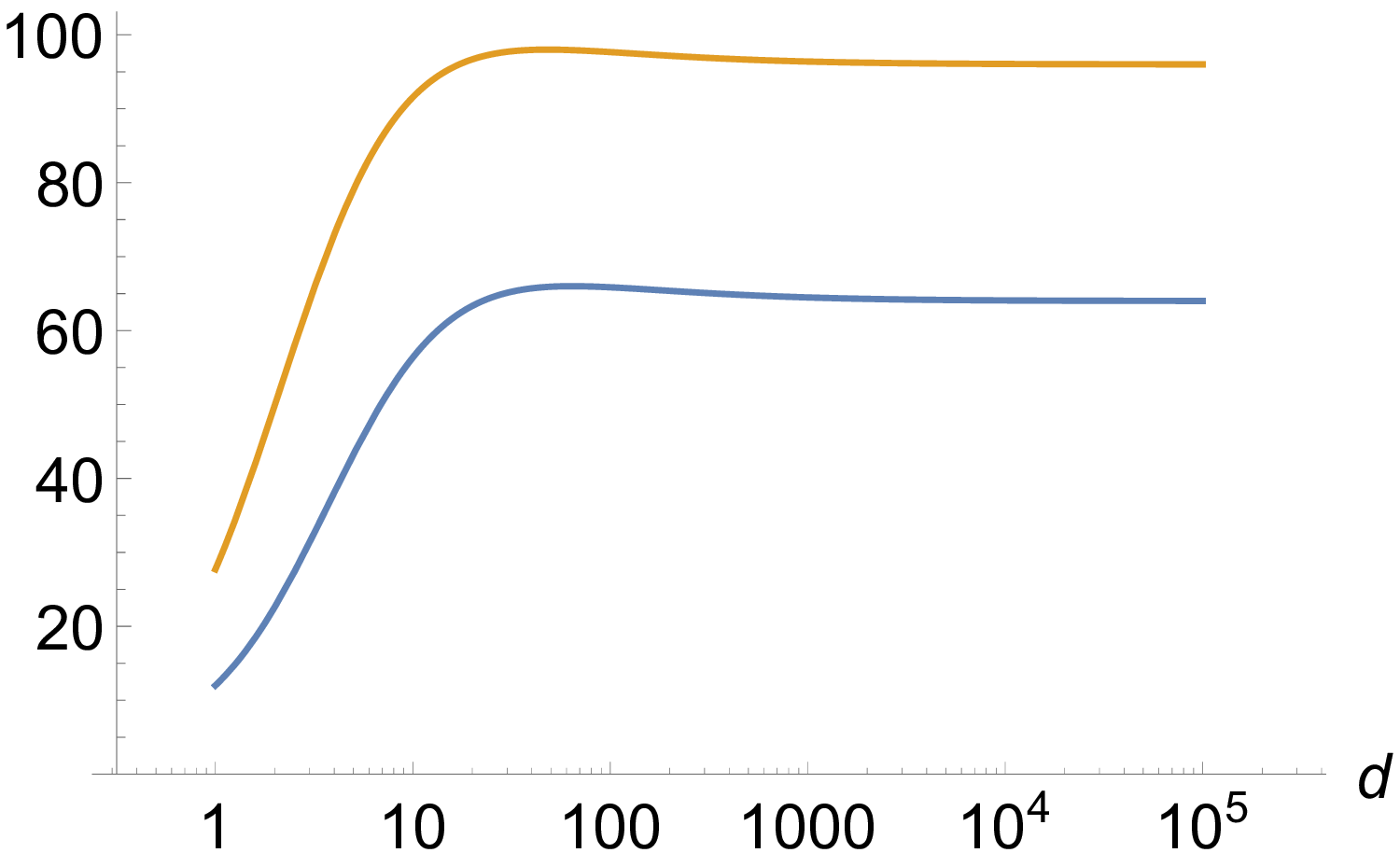}
  \caption{
    The coefficient $\beta_d$ of \cref{thm:uni:tilde} (blue, for $\tilde z$)
    and $\beta'_d$ of \cref{thm:uni:breve} (orange, for $\breve z$).
    \Textcite{rks} used a constant of 256 for $\tilde z$.
  }
  \label{fig:uni-constants}
\end{figure}

For the Gaussian kernel,
$\alpha_\varepsilon \le \tfrac12 + \frac13 \varepsilon$
and $\sigma_p^2 = d / \sigma^2$;
the Bernstein bound is tighter when $\varepsilon < \frac32$.

For $\breve z$, since the embedding $\breve s$ is not shift-invariant,
we must instead place the $\varepsilon$-net on $\X^2$.
The additional noise in $\breve s$ also increases the expected Lipschitz constant
and gives looser bounds on each term in the sum,
though there are twice as many such terms.
The corresponding bound is as follows:
\begin{proposition} \label{thm:uni:breve}
Let $k$, $\X$, $\ell$, $P(\omega)$, and $\sigma_p$ be as in \cref{thm:uni:tilde}.
Define $\breve z$ by \cref{eq:z-breve},
and $\breve f(x, y) := \breve z(x)\tp \breve z(y) - k(x, y)$.
For any $\varepsilon > 0$, define
\begin{align}
\alpha'_\varepsilon &:= \min\left(1, \sup_{x, y \in \X} \tfrac14 + \tfrac18 k(2x, 2y) - \tfrac14 k(x, y)^2 + \tfrac16 \varepsilon \right),
\\
  \beta'_d &:= \left( d^\frac{-d}{d+1} + d^\frac{1}{d+1} \right) 2^\frac{5d+1}{d+1} 3^\frac{d}{d+1}
.\end{align}
Then, assuming only for the second statement that $\varepsilon \le \sigma_p \ell$,
\begin{align}
    \Pr\left( \norm{\breve f}_\infty \ge \varepsilon \right)
   &\le
    \beta'_d
    \left( \frac{\sigma_p \ell}{\varepsilon} \right)^\frac{2}{1 + \frac{1}{d}}
    \exp\left( - \frac{D \varepsilon^2}{32 (d + 1) \alpha'_\varepsilon} \right)
\\ &\le
    98
    \left( \frac{\sigma_p \ell}{\varepsilon} \right)^2
    \exp\left( - \frac{D \varepsilon^2}{32 (d + 1)} \right)
.\end{align}
Thus, we can achieve an embedding with pointwise error no more than $\varepsilon$
with probability at least $1 - \delta$ as long as
\[
    D
    \ge
    \frac{32 (d + 1) \alpha_\varepsilon'}{\varepsilon^2}
    \left[
      \frac{2}{1 + \frac{1}{d}}
      \log \frac{\sigma_p \ell}{\varepsilon}
      + \log \frac{\beta_d'}{\delta}
    \right]
.\]
\end{proposition}
$\beta'_{48} = 98$, and $\lim_{d \to \infty} \beta'_d = 96$,
also shown in \cref{fig:uni-constants}.
The full proof is given in
\ifbool{shortversion}{%
  Appendix A.2 (in the supplemental material).
}{%
  \cref{proof:uni:breve}.
}

For the Gaussian kernel,
$\alpha'_\varepsilon \le \tfrac14 + \tfrac16 \varepsilon$,
so that the Berstein bound is essentially always superior.

\subsubsection{Expected Max Error} \label{sec:error:uni:exp}

Noting that 
$\E \norm{f}_\infty = \int_0^\infty \Pr\left( \norm{f}_\infty \ge \varepsilon \right) \,\ud\varepsilon$,
one could consider bounding $\E \norm{f}_\infty$
via \cref{thm:uni:tilde,thm:uni:breve}.
Unfortunately, that integral diverges on $(0, \gamma)$ for any $\gamma > 0$.
If we instead integrate the minimum of that bound and 1,
the result depends on a solution to a transcendental equation,
so analytical manipulation is difficult.

We can, however, use
a slight generalization of Dudley's entropy integral \parencite*{dudley}
to obtain the following bound:
\begin{proposition} \label{thm:exp-sup-tilde}
Let $k$, $\X$, $\ell$, and $P(\omega)$ be as in \cref{thm:uni:tilde}.
Define $\tilde z$ by \eqref{eq:z-tilde},
and $\tilde f(x, y) := \tilde z(x)\tp \tilde{z}(y) - k(x, y)$.
Let $\X_\Delta := \{ x - y \mid x, y \in \X \}$;
suppose $k$ is $L$-Lipschitz on $\X_\Delta$.
Let
$
    R := \E \max_{i = 1, \dots, \frac{D}{2}} \norm{\omega_i}
$.
Then
\[
    \E\left[ \norm{\tilde f}_\infty \right]
  \le \frac{24 \gamma \sqrt{d} \ell}{\sqrt{D}} (R + L)
\]
where $\gamma \approx 0.964$.
\end{proposition}

The proof is given in 
\ifbool{shortversion}{Appendix A.3}{\cref{proof:exp:tilde}}.
In order to apply the method of \textcite{dudley},
we must work around $\norm{\omega_i}$
(which appears in the covariance of the error process)
being potentially unbounded.
To do so, we bound a process with truncated $\norm{\omega_i}$,
and then relate that bound to $\tilde f$.

For the Gaussian kernel,
$L = 1 / (\sigma \sqrt{e})$
and\footnote{%
    By the Gaussian concentration inequality \parencite[Theorem~5.6]{boucheron},
    each $\norm{\omega} - \E \norm \omega$ is sub-Gaussian with variance factor $\sigma^{-2}$;
    the claim follows from their Section~2.5.
}
\begin{align}
       R
  &\le \left( \sqrt{2} \, \frac{\Gamma\left( (d+1) / 2 \right)}{\Gamma\left( d/2 \right)}
     + \sqrt{2 \log\left( D/2 \right)} \right) / \sigma
\\&\le \left( \sqrt{d}
     + \sqrt{2 \log\left( D/2 \right)} \right) / \sigma
.\end{align}
Thus
$\E \norm{\tilde f}_\infty$ is less than
\begin{gather}
    \frac{24 \gamma \sqrt{d} \, \ell}{\sqrt{D} \, \sigma} \left( e^{-1/2} + \sqrt{d} + \sqrt{2 \log(D/2)} \right)
\label{eq:exp-bound:gaussian}
.\end{gather}

We can also prove an analogous bound for the $\tilde z$ features:
\begin{proposition} \label{thm:exp-sup-breve}
  Let $k, \X, \ell$, and $P(\omega)$ be as in \cref{thm:uni:tilde}.
  Define $\breve z$ by \cref{eq:z-breve},
  and $\breve f(x, y) := \breve z(x)\tp \breve z(y) - k(x, y)$.
  Suppose $k(\Delta)$ is $L$-Lipschitz.
  Let $R := \E \max_{i=1, \dots, D} \norm{\omega_i}$.
  Then, for $\X$ and $D$ not extremely small,
  \[
    \E\left[ \norm{\breve f}_\infty \right]
    \le 
    \frac{48 \gamma'_{\X} \ell \sqrt{d}}{\sqrt{D}} (R + L)
  \]
  where $0.803 < \gamma'_\X < 1.542$.
  See
  \ifbool{shortversion}{Appendix A.4}{\cref{proof:exp:breve}}
  for details on $\gamma'_\X$ and the ``not extremely small'' assumption.
\end{proposition}

The proof is given in
\ifbool{shortversion}{Appendix A.4}{\cref{proof:exp:breve}}.
It is similar to that for \cref{thm:exp-sup-tilde},
but the lack of shift invariance increases some constants
and otherwise slightly complicates matters.
Note also that the $R$ of \cref{thm:exp-sup-breve}
is somewhat larger than that of \cref{thm:exp-sup-tilde}.

\subsubsection{Concentration About Mean} \label{sec:error:uni:conc}

Bousquet's inequality \parencite*{bousquet2002}
can be used to show exponential concentration of $\sup f$ about its mean.

We consider $\tilde f$ first.
Let
\[
    \tilde f_{\omega}(\Delta) := \frac{2}{D} \left( \cos(\omega\tp\Delta) - k(\Delta) \right)
,\]
so $f(\Delta) = \sum_{i=1}^{D/2} \tilde f_{\omega_i}(\Delta)$.
Define the ``wimpy variance'' of $\tilde f/2$ (which we use so that $\abs{\tilde f/2} \le 1$) as
\begin{align}
    \sigma_{\tilde f / 2}^2
  :&= \sup_{\Delta \in \X_\Delta} \sum_{i=1}^{D/2} \Var\left[ \tfrac12 \tilde f_{\omega_i}(\Delta) \right]
\\&= \frac{1}{D} \sup_{\Delta \in \X_\Delta} \left[ 1 + k(2 \Delta) - 2 k(\Delta)^2 \right]
\\&=: \frac{1}{D} \sigma_w^2
,\end{align}
using \cref{eq:var-cos-wd}.
Clearly $1 \le \sigma_w^2 \le 2$;
for the Gaussian kernel, it is 1.

\begin{proposition} \label{thm:uni:conc}
Let $k$, $\X$, and $P(\omega)$ be as in \cref{thm:uni:tilde},
and $\tilde z$ be defined by \cref{eq:z-tilde}.
Let $\tilde f(\Delta) = \tilde z(x)\tp \tilde z(y) - k(\Delta)$ for $\Delta = x - y$,
and $\sigma_{w}^2 := \sup_{\Delta \in \X_\Delta} 1 + k(2 \Delta) - 2 k(\Delta)^2$.
Then
\begin{multline}
    \Pr\left(
\norm{\tilde f}_\infty
- \E \norm{\tilde f}_\infty
\ge
\varepsilon
\right)
\\\le
    2 \exp\left( - \frac{D \varepsilon^2}{D \, \E \norm{\tilde f}_\infty + \frac12 \sigma_{w}^2 + \frac{D \varepsilon}{6}} \right)
.\end{multline}
\end{proposition}
\begin{proof}
We use the Bernstein-style form of Theorem~12.5 of \textcite{boucheron}
on $\tilde f(\Delta) / 2$ to obtain 
that
$\Pr\left( \sup \tilde f - \E \sup \tilde f \ge \varepsilon \right)$
is at most
\[
    \exp\left( - \frac{\varepsilon^2}{\E \sup \tilde f + \tfrac12 \sigma_{\tilde f / 2}^2 + \frac\varepsilon6} \right)
.\]
The same holds for $- \tilde f$,
and $\E \sup \tilde f \le \E \sup \norm{f}_\infty$,
$\E \sup (-\tilde f) \le \E \sup \norm{f}_\infty$.
The claim follows by a union bound.
\end{proof}

A bound on the lower tail, unfortunately, is not available in the same form.

For $\breve f$,
note $\abs{\breve f} \le 3$,
so we use $\breve f / 3$.
Letting $\breve f_\omega := \tfrac1D (\cos(\omega\tp\Delta) - k(\Delta))$,
we have
$\sigma_{\breve f / 3}^2 = \frac{1}{18 D} (\sigma_w^2 + 1)$.
Thus the same argument gives us:
\begin{proposition} \label{thm:uni:conc:breve}
Let $k$ and $\X$ be as in \cref{thm:uni:tilde},
with $P(\omega)$ defined as there.
Let $\breve z$ be as in \cref{eq:z-breve},
$\tilde f(x, y) = \tilde z(x)\tp \tilde z(y) - k(x, y)$,
and define $\sigma_{w}$ as above.
Then
\begin{multline}
    \Pr\left( \norm{\breve f}_\infty - \E \norm{\breve f}_\infty \ge \varepsilon \right)
    \\ \le 2 \exp\left( - \frac{ D \varepsilon^2}{\frac{4}{9} D \, \E \norm{\breve f}_\infty + \frac{1}{81} (\sigma_w^2 + 1) + \frac{2}{27} D \varepsilon} \right)
.\end{multline}
\end{proposition}

Note that \cref{thm:uni:conc:breve} actually gives a somewhat tighter concentration than
\cref{thm:uni:conc}.
This is most likely because,
between the space of possible errors being larger and the higher base variance illustrated in \cref{fig:rbf-var},
the $\breve f$ error function has more ``opportunities'' to achieve its maximal error.
The experimental results (\cref{fig:survivals}) show that, at least in one case,
$\norm{\breve f}_\infty$ does concentrate about its mean more tightly,
but that mean is enough higher than that of $\norm{\tilde f}_\infty$
that $\norm{\breve f}_\infty$ stochastically dominates $\norm{\tilde f}_\infty$.

\subsection{\texorpdfstring{$L_2$}{L2} ERROR BOUND} \label{sec:error:l2}

$L_\infty$ bounds provide useful guarantees,
but are very strict.
It can also be useful to consider a less stringent error measure.
Let $\mu$ be a $\sigma$-finite measure on $\X \times \X$;
define
\[
    \Norm{f}_\mu^2 := 
    \int_{\X^2} f(x, y)^2 \,\ud\mu(x, y)
\label{eq:norm-mu}
.\]

First, we have that
\begin{align}
       \E \norm{\tilde f}_\mu^2
  & =  \E \int_{\X^2} \tilde f(x, y)^2 \,\ud\mu(x, y)
\\& =  \int_{\X^2} \E\, \tilde f(x, y)^2 \,\ud\mu(x, y)
\label{eq:l2:integral-exchange}
\\& =  \int_{\X^2} \frac{1}{D} \left[ 1 + k(2x, 2y) - 2 k(x, y)^2 \right] \,\ud\mu(x, y)
\\& =  \frac{1}{D} \left[ \mu(\X^2) + \int_{\X^2} k(2x, 2y) \,\ud\mu(x, y) - 2 \Norm{k}_\mu^2 \right]
\\     \E \norm{\breve f}_\mu^2
  & =  \frac{1}{D} \left[ \mu(\X^2) + \frac12 \int_{\X^2} k(2x, 2y) \,\ud\mu(x, y) - \Norm{k}_\mu^2 \right]
\end{align}
where \cref{eq:l2:integral-exchange} is justified by Tonelli's theorem.

If $\mu = P_X \times P_Y$ is a joint distribution of independent variables,
then $\int_{\X^2} k(2x, 2y) \,\ud\mu(x, y) = \acro{mmk}(P_{2X}, P_{2Y})$,
where \acro{mmk} is the mean map kernel (see \cref{sec:downstream:mmd}).
Likewise, $\norm{k}_\mu^2 = \acro{mmk}(P_X, P_Y)$ using the kernel $k^2$.%
\footnote{$k^2$ is also a \acro{psd} kernel, by the Schur product theorem.}

Viewing $\norm{\tilde f}_\mu$ as a function of $\omega_1, \dots, \omega_{D/2}$,
changing $\omega_i$ to a different $\hat\omega_i$
changes the value of $\norm{\tilde f}_\mu$ by at most 
$4 \frac{4 D + 1}{D^2} \mu(\X^2)$;
this can be seen by simple algebra and is shown in
\ifbool{shortversion}{Appendix B.1}{\cref{proof:l2:diff:tilde}}.
Thus \textcite{mcdiarmid} gives us an exponential concentration bound:
\begin{proposition} \label{thm:l2:tilde}
    Let k be a continuous shift-invariant positive-definite function
    $k(x, y) = k(\Delta)$ defined on $\X \subseteq \R^d$,
    with $k(0) = 1$.
    Let $\mu$ be a $\sigma$-finite measure on $\X^2$,
    and define $\norm{\cdot}_\mu^2$ as in \cref{eq:norm-mu}.
    Define $\tilde z$ as in \cref{eq:z-tilde}
    and let $\tilde f(x, y) = \tilde z(x)\tp \tilde z(y) - k(x, y)$.
    Let $\mathcal{M} := \mu(\X^2)$.
    Then
    \begin{align}
      \Pr\left( \Abs{\norm{\tilde f}_\mu^2 - \E \norm{\tilde f}_\mu^2} \ge \varepsilon \right)
     &\le 2 \exp\left( \frac{- D^3 \varepsilon^2}{8 (4 D + 1)^2 \,\mathcal{M}^2} \right)
   \\&\le 2 \exp\left( \frac{- D \varepsilon^2}{200 \, \mathcal{M}^2} \right)
    .\end{align}
\end{proposition}
The second version of the bound is simpler, but somewhat looser for $D \gg 1$;
asymptotically, the coefficient of the denominator becomes 128.

Similarly, the variation of $\norm{\breve f}_\mu$
is bounded by at most $32 \frac{D + 1}{D^2} \mu(\X^2)$
(shown in \ifbool{shortversion}{Appendix B.2}{\cref{proof:l2:diff:breve}}).
Thus:
\begin{proposition} \label{thm:l2:breve}
    Let $k$, $\mu$, $\norm{\cdot}_\mu$, and $\mathcal M$ be as in \cref{thm:l2:tilde}.
    Define $\breve z$ as in \cref{eq:z-breve}
    and let $\breve f(x, y) = \breve z(x)\tp \breve z(y) - k(x, y)$.
    Then
    \begin{align}
      \Pr\left( \Abs{\norm{\breve f}_\mu^2 - \E \norm{\breve f}_\mu^2} \ge \varepsilon \right)
     &\le 2 \exp\left( \frac{- D^3 \varepsilon^2}{512 (D + 1)^2 \,\mathcal{M}^2} \right)
   \\&\le 2 \exp\left( \frac{- D \varepsilon^2}{2048 \, \mathcal{M}^2} \right)
    .\end{align}
\end{proposition}
The cost of a simpler dependence on $D$ is higher here;
the asymptotic coefficient of the denominator is $512$.

\section{DOWNSTREAM ERROR} \label{sec:downstream}

Rahimi and Recht \parencite*{rahimi-nips,rahimi-allerton}
give a bound on the $L_2$ distance between any given function in the
reproducing kernel Hilbert space (\acro{rkhs})
induced by $k$
and the closest function in the \acro{rkhs} of $s$:
results invaluable for the study of learning rates.
In some situations, however,
it is useful to consider not the learning-theoretic convergence of hypotheses to the assumed ``true'' function,
but rather directly consider the difference in predictions
due to using the $z$ embedding
instead of the exact kernel $k$.

\subsection{KERNEL RIDGE REGRESSION} \label{sec:downstream:krr}
We first consider kernel ridge regression \parencite[\acro{KRR};][]{saunders:krr}.
Suppose we are given $n$ training pairs
$(x_i, y_i) \in \R^d \times \R$
as well as a regularization parameter $\lambda = n \lambda_0 > 0$.
We construct the training Gram matrix $K$ by $K_{ij} = k(x_i, x_j)$.
\acro{krr} gives predictions $h(x) = \alpha\tp k_x$,
where $\alpha = (K + \lambda I)^{-1} y$
and $k_x$ is the vector with $i$th component $k(x_i, x)$.\footnote{%
    If a bias term is desired,
    we can use $k'(x, x') = k(x, x') + 1$
    by appending a constant feature $1$ to the embedding $z$.
    Because this change is accounted for exactly,
    it affects the error analysis here only in that we must use
    $\sup \abs{k(x, y)} \le 2$,
    in which case the first factor of \cref{eq:krr:uni}
    becomes $(\lambda_0 + 2) / \lambda_0^2$.
}
When using Fourier features, one would not use $\alpha$,
but instead a primal weight vector $w$;
still, it will be useful for us to analyze the situation in the dual.

Proposition~1 of \textcite{cortes:approx} bounds the change in \acro{KRR} predictions
from approximating the kernel matrix $K$ by $\hat K$,
in terms of $\norm{\hat K - K}_2$.
They assume, however, that the kernel evaluations at test time $k_x$
are unapproximated,
which is certainly not the case when using Fourier features.
We therefore extend their result to \cref{thm:krr:approx}
before using it to analyze the performance of Fourier features.

\begin{proposition} \label{thm:krr:approx}
    Given a training set $\left\{ (x_i, y_i) \right\}_{i=1}^n$,
    with $x_i \in \R^d$ and $y_i \in \R$,
    let $h(x)$ denote the result of kernel ridge regression
    using the \acro{psd} training kernel matrix $K$
    and test kernel values $k_x$.
    Let $\hat h(x)$ be the same using a \acro{psd} approximation
    to the training kernel matrix $\hat K$
    and test kernel values $\hat k_x$.
    Further assume that the training labels are centered, $\sum_{i=1}^n y_i = 0$,
    and let $\sigma_y^2 := \frac{1}{n} \sum_{i=1}^n y_i^2$.
    Also suppose $\norm{k_x}_\infty \le \kappa$.
    Then:
    \[
        \Abs{h'(x) - h(x)}
        \le \frac{\sigma_y}{\sqrt{n} \lambda_0} \norm{\hat k_x - k_x}
          + \frac{\kappa \sigma_y}{n \lambda_0^2} \norm{\hat K - K}_2
    .\]
\end{proposition}
\begin{proof}
    Let $\alpha = (K + \lambda I)^{-1} y$,
    $\hat \alpha = (\hat K + \lambda I)^{-1} y$.
    Thus,
    using $\hat{M}^{-1} - M^{-1} = - \hat{M}^{-1} (\hat{M} - M) M^{-1}$,
    we have
    \begin{align}
           \hat\alpha - \alpha
      &  = - (\hat K + \lambda I)^{-1} (\hat K - K) (K + \lambda I)^{-1} y
    \\     \norm{\hat\alpha - \alpha}
      &\le \norm{(\hat K + \lambda I)^{-1}}_2
           \norm{\hat K - K}_2
           \norm{(K + \lambda I)^{-1}}_2
           \norm{y}
    \\&\le \frac{1}{\lambda^2} \norm{\hat K - K}_2 \, \norm y
    \end{align}
    since the smallest eigenvalues of $K + \lambda I$ and $\hat K + \lambda I$
    are at least $\lambda$.
    Since $\norm{k_x} \le \sqrt{n} \kappa$
    and $\norm{\hat \alpha} \le \norm{y} / \lambda$:
    \begin{align}
           \abs{\hat h(x) - h(x)}
      &  = \abs{\hat\alpha\tp \hat k_x - \alpha\tp k_x}
    \\&  = \abs{\hat\alpha\tp (\hat k_x - k_x) + (\hat\alpha - \alpha)\tp k_x}
    \\&\le \norm{\hat\alpha} \norm{\hat k_x - k_x} + \norm{\hat\alpha - \alpha} \norm{k_x}
    \\&\le \frac{\norm y}{\lambda} \norm{\hat k_x - k_x}
         + \frac{\sqrt{n} \kappa \norm y}{\lambda^2} \norm{\hat K - K}_2
    .\end{align}
    The claim follows from $\lambda = n \lambda_0$,
    $\norm y = \sqrt{n} \sigma_y$.
\end{proof}

Suppose that, per the uniform error bounds of \cref{sec:error:uni},
$\sup \Abs{k(x, y) - s(x, y)} \le \varepsilon$.
Then
$\norm{\hat k_x - k_x} \le \sqrt{n} \varepsilon$
and
$\norm{\hat K - K}_2 \le \norm{\hat K - K}_F \le n \varepsilon$,
and \Cref{thm:krr:approx} gives
\begin{align}
       \Abs{\hat h(x) - h(x)}
  &\le \frac{\sigma_y}{\sqrt{n} \lambda_0} \sqrt{n} \varepsilon
     + \frac{\sigma_y}{n \lambda_0^2} n \varepsilon
\\&\le \frac{\lambda_0 + 1}{\lambda_0^2} \sigma_y \varepsilon
\label{eq:krr:uni}
.\end{align}
Thus
\begin{align}
       \Pr\left( \Abs{h'(x) - h(x)} \ge \varepsilon \right)
  &\le \Pr\left( \norm{f}_\infty \ge \frac{\lambda_0^2 \varepsilon}{(\lambda_0 + 1) \sigma_y} \right)
.\end{align}
which we can bound
with \cref{thm:uni:tilde} or \ref{thm:uni:breve}.
We can therefore guarantee $\abs{h(x) - h'(x)} \le \varepsilon$
with probability at least $\delta$ 
if
\begin{multline}
    D
    = \Omega\left(
        d \,
        \left(\frac{(\lambda_0 + 1) \sigma_y}{\lambda_0^2 \, \varepsilon}\right)^2
\right.\\\left.
        \left[
            \log \delta
            + \log \frac{\lambda_0^2 \varepsilon}{(\lambda_0 + 1) \sigma_y}
            - \log \sigma_p \ell
        \right]
    \right)
.\end{multline}
Note that this rate does not depend on $n$.
If we want $h'(x) \to h(x)$
at least as fast
as $h(x)$'s convergence rate of $O(1/\sqrt{n})$ \parencite{bousquet:stability},
ignoring the logarithmic terms,
we thus need $D$ to be linear in $n$,
matching the conclusion of \textcite{rahimi-nips}.

\subsection{SUPPORT VECTOR MACHINES} \label{sec:downstream:svm}
Consider a Support Vector Machine (\acro{svm}) classifier with no offset,
such that $h(x) = w\tp \Phi(x)$ for a kernel embedding $\Phi(x) : \X \to \mathcal{H}$
and $w$ is found by
\[
    \argmin_{w \in \mathcal H}
        \frac12 \norm{w}^2
        + \frac{C_0}{n} \sum_{i=1}^n \max\left(0, 1 - y_i \langle w, \Phi(x_i) \rangle \right)
\]
where $\{ (x_i, y_i) \}_{i=1}^n$ is our training set
with $y_i \in \{-1, 1\}$,
and the decision function is
$h(x) = \langle w, \Phi(x) \rangle$.\footnote{%
    We again assume there is no bias term for simplicity;
    adding a constant feature again changes the analysis only in that it
    makes the $\kappa$ of \cref{thm:svm:approx} 2 instead of 1.
}
For a given $x$,
\textcite{cortes:approx}
consider an embedding in $\mathcal H = \R^{n+1}$
which is equivalent on the given set of points.
They bound $\Abs{\hat h(x) - h(x)}$
in terms of $\norm{\hat K - K}_2$
in their Proposition~2,
but again assume that the test-time kernel values $k_x$ are exact.
We will again extend their result in \cref{thm:svm:approx}:
\begin{proposition} \label{thm:svm:approx}
    Given a training set
    $\{ (x_i, y_i) \}_{i=1}^n$,
    with $x_i \in \R^d$ and $y_i \in \{ -1, 1 \}$,
    let $h(x)$ denote the decision function of an \acro{svm} classifier
    using the \acro{psd} training matrix $K$
    and test kernel values $k_x$.
    Let $\hat h(x)$ be the same
    using a \acro{psd} approximation to the training kernel matrix $\hat K$
    and test kernel values $\hat k_x$.
    Suppose $\sup k(x, x) \le \kappa$.
    Then:
    \begin{multline}
        \abs{\hat h(x) - h(x)}
        \\\le \sqrt{2} \kappa^\frac{3}{4} C_0 \left( \norm{\hat K - K}_2 + \norm{\hat k_x - k_x} + \abs{f_x} \right)^{1/4}
        \\+ \sqrt{\kappa} C_0 \left( \norm{\hat K - K}_2 + \norm{\hat k_x - k_x} + \abs{f_x} \right)^{1/2}
    ,\end{multline}
    where $f_x = \hat k(x, x) - k(x, x)$.
\end{proposition}
\begin{proof}
    Use the setup of Section~2.2 of \textcite{cortes:approx}.
    In particular, we will use
    $\norm{w} \le \sqrt{\kappa} C_0$
    and their (16-17):
    \begin{gather}
        \Phi(x_i) = K_{x}^{1/2} e_i
        \\
        \Norm{\hat w - w}^2 \le 2 C_0^2 \sqrt{\kappa} \norm{\hat K_{x}^{1/2} - K_{x}^{1/2}}
    ,\end{gather}
    where $K_{x} = \begin{bmatrix} K & k_x \\ k_x\tp & k(x, x) \end{bmatrix}$
    and $e_i$ the $i$th standard basis.

    Further, Lemma~1 of \textcite{cortes:approx} says that
    $\norm{\hat K_{x}^{1/2} - K_{x}^{1/2}}_2 \le \norm{\hat K_{x} - K_{x}}_2^{1/2}$.
    Let $f_x := \hat k(x, x) - k(x, x)$;
    Then, by Weyl's inequality for singular values,
    \[
        \Norm{\begin{bmatrix}
            \hat K - K
          & \hat k_x - k_x
         \\ \hat k_x\tp - k_x\tp
          & f_x
        \end{bmatrix}}_2
        \le
        \norm{\hat K - K}_2
        + \norm{\hat k_x - k_x}
        + \Abs{f_x}
    .\]
    Thus
    \begin{align}
       \lvert \hat h&(x) - h(x) \rvert
\\\quad&  = \Abs{(\hat w - w)\tp \hat\Phi(x) + w\tp (\hat\Phi(x) - \Phi(x))}
\\&\le \norm{\hat w - w} \norm{\hat\Phi(x)} + \norm{w} \norm{\hat\Phi(x) - \Phi(x)}
\\&\le \sqrt{2} \kappa^\frac{1}{4} C_0 \norm{\hat K_{x}^{1/2} - K_{x}^{1/2}}_2^{1/2} \sqrt\kappa
  \\&\quad + \sqrt{\kappa} C_0 \norm{(\hat K_{x}^{1/2} - K_{x}^{1/2}) e_{n+1}}
\\&\le \sqrt{2} \kappa^\frac{3}{4} C_0 \norm{\hat K_{x} - K_{x}}_2^{1/4}
  \\&\quad + \sqrt{\kappa} C_0 \norm{\hat K_{x} - K_{x}}^{1/2}
\\&\le \sqrt{2} \kappa^\frac{3}{4} C_0 \left( \norm{\hat K - K}_2 + \norm{\hat k_x - k_x} + \abs{f_x} \right)^{1/4}
  \\&\quad + \sqrt{\kappa} C_0 \left( \norm{\hat K - K}_2 + \norm{\hat k_x - k_x} + \abs{f_x} \right)^{1/2}
    \end{align}
    as claimed.
\end{proof}

Suppose that
$\sup \abs{k(x, y) - s(x, y)} \le \varepsilon$.
Then, as in the last section,
$\norm{\hat k_x - k_x} \le \sqrt n \varepsilon$
and $\norm{\hat K - K}_2 \le n \varepsilon$.
Then,
letting $\gamma$ be $0$ for $\tilde z$ and $1$ for $\breve z$,
\cref{thm:svm:approx} gives
\begin{align}
    \abs{\hat h(x) - h(x)}
  &\le \sqrt{2} C_0 \left( n + \sqrt{n} + \gamma \right)^{1/4} \varepsilon^{1/4}
\\&\qquad
      + C_0 \left( n + \sqrt{n} + \gamma \right)^{1/2} \varepsilon^{1/2}
.\end{align}
Then $\abs{\hat h(x) - h(x)} \ge u$
only if
\[
    \varepsilon \le
    \frac{2 C_0^2 + 4 C_0 u + u^2 - 2 (C_0 + u) \sqrt{C_0 (C_0 + 2 u)}}%
         {C_0^2 (n + \sqrt n + \gamma)}
.\]
This bound has the unfortunate property of requiring the approximation to
be \emph{more} accurate as the training set size increases,
and thus can prove only a very loose upper bound on the number of features
needed to achieve a given approximation accuracy,
due to the looseness of \cref{thm:svm:approx}.
Analyses of generalization error in the induced $\acro{rkhs}$,
such as \textcite{rahimi-nips,yang:nystroem-vs-fourier},
are more useful in this case.

\subsection{MAXIMUM MEAN DISCREPANCY} \label{sec:downstream:mmd}
Another area of application for random Fourier embeddings is
to the mean embedding of distributions,
which uses some kernel $k$
to represent a probability distribution $P$
in the \acro{rkhs} induced by $k$ as
$\varphi(P) = \E_{x \sim P}\left[ k(x, \cdot) \right]$.
For samples $\{ X_i \}_{i=1}^n \sim P$ and $\{ Y_j \}_{j=1}^m \sim Q$,
we can estimate the inner product in the embedding space,
the \emph{mean map kernel} (\acro{mmk}), by
\[
    \acro{mmk}(X, Y)
    := \frac{1}{n m} \sum_{i=1}^n \sum_{j=1}^m k(X_i, Y_j)
    \approx \left\langle \varphi(P), \varphi(Q) \right\rangle
.\]
The distance $\norm{\varphi(P) - \varphi(Q)}$ is known as the
\emph{maximum mean discrepancy} (\acro{mmd}),
which can be estimated with:
\begin{multline}
    \norm{\varphi(P) - \varphi(Q)}^2
    \\= 
      \left\langle \varphi(P), \varphi(P) \right\rangle
    + \left\langle \varphi(Q), \varphi(Q) \right\rangle
    - 2 \left\langle \varphi(P), \varphi(Q) \right\rangle
.\end{multline}
$\acro{mmk}(X, X)$ is a biased estimator,
because of the $k(X_i, X_i)$ and $k(Y_i, Y_i)$ terms;
removing them gives an unbiased estimator \parencite{gretton:two-sample}.
The \acro{mmk} can be used in standard kernel methods
to perform learning on probability distributions,
such as when images are treated as sets of local patch descriptors \parencite{muandet:smm}
or documents as sets of word descriptors \parencite{yoshikawa:latent-smm}.
The \acro{mmd} has strong applications to two-sample testing,
where it serves as the statistic for testing the hypothesis that $X$ and $Y$
are sampled from the same distribution \parencite{gretton:two-sample};
this has applications in, for example,
comparing microarray data from different experimental situations
or in matching attributes when merging databases.

The \acro{mmk} estimate can clearly be approximated with an explicit embedding:
if $k(x, y) \approx z(x)\tp z(y)$,
\begin{align}
    \acro{MMK}_z(X, Y)
  &= \frac{1}{nm} \sum_{i=1}^n \sum_{j=1}^m z(X_i)\tp z(Y_j)
\\&= \left( \frac{1}{n} \sum_{i=1}^n z(X_i) \right)\tp
     \left( \frac{1}{m} \sum_{j=1}^m z(Y_j) \right)
\\&= \bar{z}(X)\tp \bar{z}(Y)
.\end{align}
Thus the biased estimator of $\acro{mmk}(X, X)$ is just
$\Norm{\bar{z}(X)}^2$;
the unbiased estimator is
\[
\frac{n^2}{n^2 - n} \left(\Norm{\bar{z}(X)}^2 - \frac{1}{n^2}\sum_{i=1}^n \Norm{z(X_i)}^2\right)
\]
When $z(x)\tp z(x) = 1$, as with $\tilde z$,
this simplifies to
$\frac{n}{n-1} \Norm{\bar{z}(X)}^2 - \frac{1}{n-1}$.
When that is not necessarily true,
as with $\breve z$,
that simplification holds only in expectation.

This has been noticed a few times in the literature,
e.g.\ by \textcite{li:relative-outliers}.
\Textcite{gretton:two-sample} gives different linear-time test statistics
based on subsampling the sum over pairs;
this version avoids reducing the amount of data used in favor of approximating the kernel.
Additionally, when using the \acro{mmk} in a kernel method
this approximation allows the use of linear solvers,
whereas the other linear approximations must still perform some pairwise computation.
\Textcite{zhao:fastmmd} compare the empirical performance of
an approximation equivalent to $\breve z$
against other linear-time approximations for two-sample testing.
They find it is slower than the \acro{mmd}-linear approximation but far more accurate,
while being more accurate and comparable in speed to a block-based $B$-test \parencite{zaremba:b-test}.

\Textcite{zhao:fastmmd} also state a simple uniform error bound on the quality of this approximation.
Specifically,
since we can write
$\Abs{\acro{mmk}_z(X, Y) - \acro{mmk}(X, Y)}$
as the mean of $\Abs{f(X_i, Y_j)}$,
uniform error bounds on $f$ apply directly to $\acro{mmk}_z$,
including to the unbiased version of $\acro{mmk}_z(X, X)$.
Moreover,
since
$\acro{MMD}^2(X, Y) = \acro{MMK}(X, X) + \acro{MMK}(Y, Y) - 2 \acro{MMK}(X, Y)$,
its error is at most $4$ times $\norm{f}_\infty$.
The advantage of this bound is that it applies uniformly to all sample sets
on the input space $\X$,
which is useful when we use $\acro{MMK}$ for a kernel method.

For a single two-sample test, however, we can get a tighter bound.
Consider $X$ and $Y$ fixed for now.
Note that $\E \acro{mmk}_z(X, Y) = \acro{mmk}(X, Y)$, by linearity of expectation.
The variance of $\acro{mmk}_z(X, Y)$ is exactly
\[
    \frac{1}{n^2 m^2} \sum_{i, j} \sum_{i', j'} \Cov\left( s(X_i, Y_j), s(X_{i'}, Y_{j'}) \right)
    \label{eq:mmd:variance}
,\]
which can be evaluated using the formulas of \cref{sec:error:variance}
and so, viewed only as a function of $D$, is $O(1 / D)$.
Alternatively,
we can use a bounded difference approach:
viewing $\acro{mmk}_{\tilde z}(X, Y)$ as a function of the $\omega_i$s,
changing $\omega_i$ to $\hat\omega_i$ changes the \acro{mmk} estimate by
\[
   \Abs{
     \frac{1}{n m} \sum_{i=1}^n \sum_{j=1}^m
        \frac{2}{D} \left(
            \cos(\hat\omega_i\tp (X_i - Y_j))
          - \cos(    \omega_i\tp (X_i - Y_j))
        \right)
   }
,\]
which is at most $4 / D$.
The bound for $\breve z$ is in fact the same here.
Thus McDiarmid's inequality tells us that for fixed sets $X$ and $Y$
and either $z$,
\[
    \Pr\left( \Abs{\acro{mmk}_{z}(X, Y) - \acro{mmk}(X, Y)} \right)
    \le 2 \exp\left( - \tfrac{1}{8} D \varepsilon^2 \right)
.\]
Thus $\E \Abs{\acro{mmk}_z(X, Y) - \acro{mmk}(X, Y)} \le 2 \sqrt{2 \pi / D}$.
Similarly,
$\acro{mmd}_z$ can be changed by at most $16 / D$,
giving
\[
    \Pr\left( \Abs{\acro{mmd}_{z}(X, Y) - \acro{mmd}(X, Y)} \right)
    \le 2 \exp\left( - \tfrac{1}{128} D \varepsilon^2 \right)
\]
and expected absolute error of at most $8 \sqrt{2 \pi / D}$.

Now, if we consider the distributions $P$ and $Q$ to be fixed but the sample sets random,
Theorems~7 and 10 of \textcite{gretton:two-sample} give exponential convergence
bounds for the biased and unbiased population estimators of \acro{mmd},
which can easily be combined with the above bounds.
Note that this approach allows the domain $\X$ to be unbounded, unlike the other bound.
One could extend this to a bound uniform over some smoothness class of distributions
using the techniques of \cref{sec:error:uni},
though we do not do so here.

\section{NUMERICAL EVALUATION} \label{sec:experiments}

\subsection{APPROXIMATION ON AN INTERVAL}

We first conduct a detailed study of the approximations on the interval $\X = [-b, b]$.
Specifically,
we evenly spaced $1\,000$ points on $[-5, 5]$
and approximated the kernel matrix using both embeddings
at $D \in \{ 50, 100, 200, \dots, 900, 1\,000, 2\,000, \dots, 9\,000, 10\,000\}$,
repeating each trial $1\,000$ times,
estimating $\norm{f}_\infty$ and $\norm{f}_\mu$ at those points.
We do not consider $d > 1$ here,
because obtaining a reliable estimate of $\sup \abs{f}$
becomes very computationally expensive even for $d = 2$.

\Cref{fig:max-by-bw} shows the behavior of $\E \norm{f}_\infty$ as $b$ increases
for various values of $D$.
As expected, the $\tilde z$ embeddings have almost no error near $0$.
The error increases out to one or two bandwidths,
after which the curve appears approximately linear in $\ell / \sigma$,
as predicted by \cref{thm:exp-sup-tilde}.

\begin{figure}[tbh]
  \centering
  \includegraphics[width=.48\textwidth]{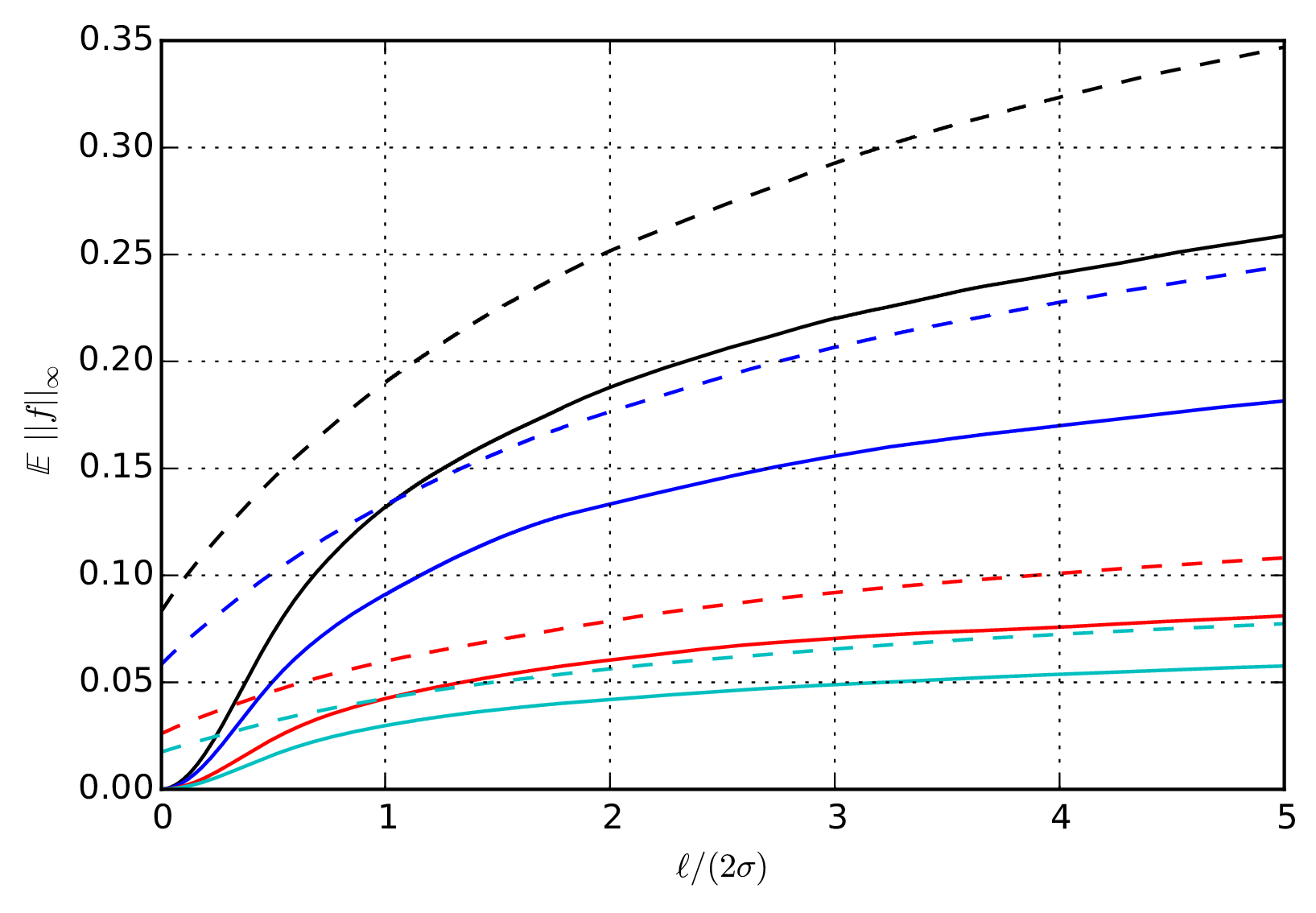}
  \caption{
    The maximum error within a given radius in $\R$,
    averaged over $1\,000$ evaluations.
    Solid lines represent $\tilde z$ and dashed lines $\breve z$;
    black is $D = 50$, blue is $D = 100$, red $D = 500$, and cyan $D = 1\,000$.
  }
  \label{fig:max-by-bw}
\end{figure}

\Cref{fig:bounds} fixes $b = 3$ and shows the expected maximal error as a function of $D$.
It also plots the expected error obtained by numerically integrating the bounds of
\cref{thm:uni:tilde,thm:uni:breve}
(using the minimum of 1 and the bound).
We can see that all of the bounds are fairly loose,
but that the first version of the bound in the propositions
(with $\beta_d$, the exponent depending on $d$, and $\alpha_\varepsilon$)
is substantially tighter than the second version when $d = 1$.

\begin{figure}[tbh]
  \centering
  \includegraphics[width=.48\textwidth]{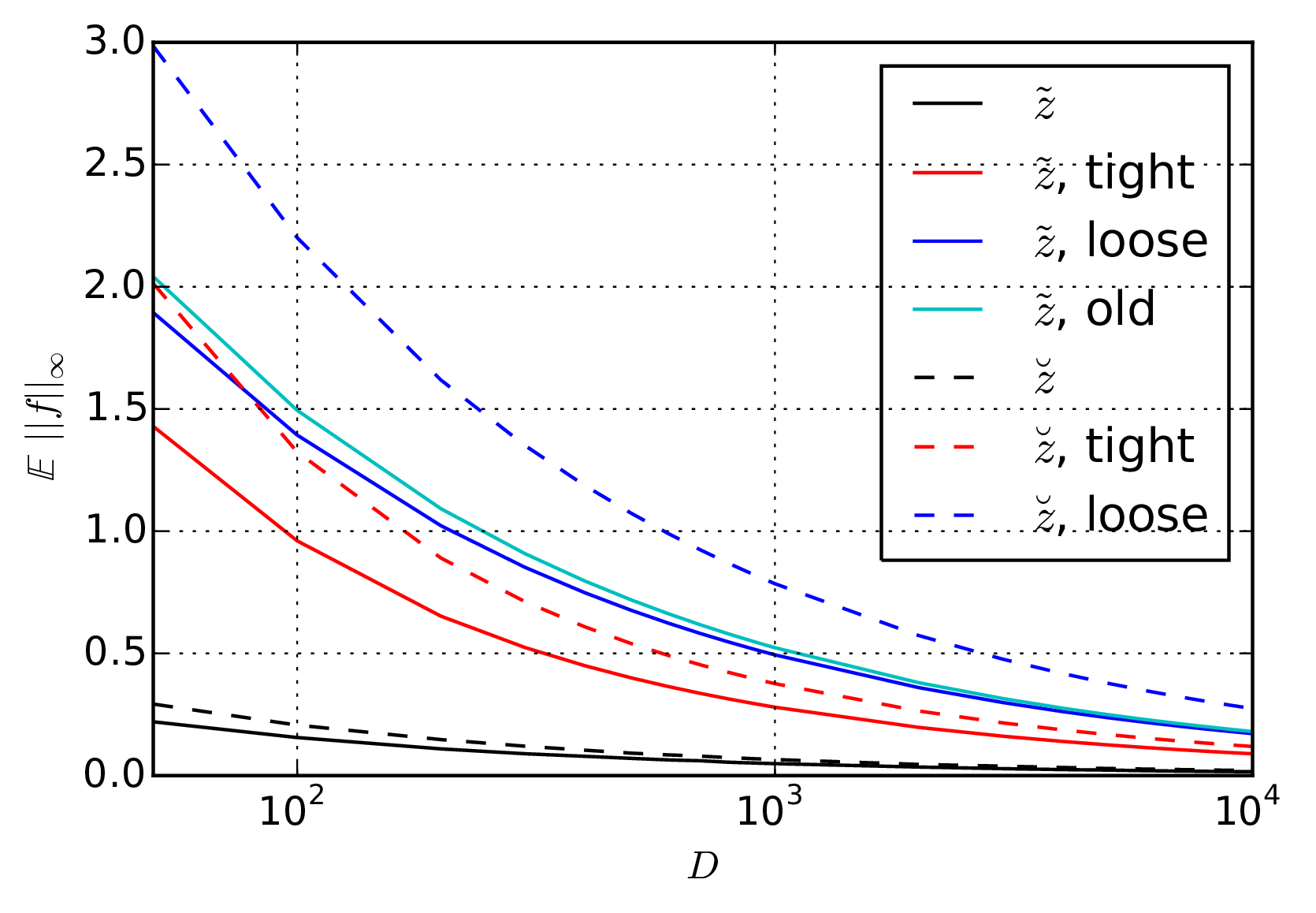}
  \caption{
    $\E \norm{f}_\infty$ for the Gaussian kernel on $[-3, 3]$ with $\sigma = 1$,
    based on the mean of $1\,000$ evaluations
    and on numerical integration of the bounds from \cref{thm:uni:tilde,thm:uni:breve}.
    (``Tight'' refers to the bound with constants depending on $d$,
    and ``loose'' the second version;
    ``old'' is the version from \textcite{rks}.)
  }
  \label{fig:bounds}
\end{figure}

The bounds on $\E \norm{f}_\infty$ of \cref{thm:exp-sup-tilde,thm:exp-sup-breve} are unfortunately too loose to show on the same plot. 
However, one important property does hold.
For a fixed $\X$, \cref{eq:exp-bound:gaussian} predicts that
$\E \norm{f}_\infty = O( {1}/{\sqrt{D}} )$.
This holds empirically:
performing linear regression of $\log \E \norm{\tilde f}_\infty$
against $\log D$
yields a model of $\E \norm{\tilde f}_\infty = e^c D^m$,
with a 95\% confidence interval for $m$ of $[-0.502, -0.496]$;
$\norm{\breve f}_\infty$ gives $[-0.503, -0.497]$.
The integrated bounds of \cref{thm:uni:tilde,thm:uni:breve}
do not fit the scaling as a function of $D$ nearly as well.

\Cref{fig:survivals} shows the empirical survival function of the max error for $D = 500$,
along with the bounds of \cref{thm:uni:tilde,thm:uni:breve}
and those of \cref{thm:uni:conc,thm:uni:conc:breve} using the empirical mean.
The latter bounds are tighter than the former for low $\varepsilon$,
especially for low $D$,
but have a lower slope.

\begin{figure}[tbh]
  \centering
  \includegraphics[width=.48\textwidth]{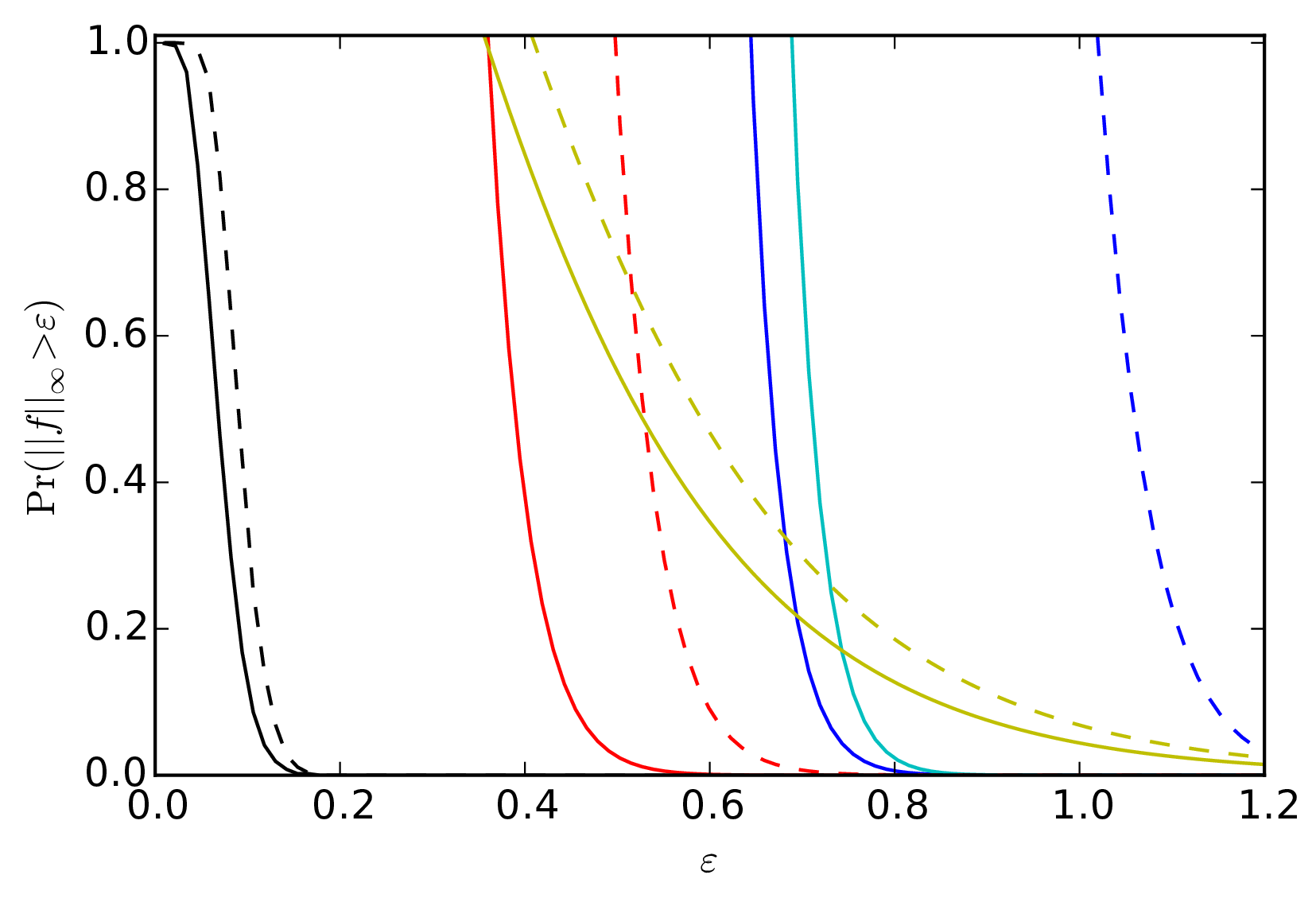}
  \caption{
    $\Pr\left( \E \norm{f}_\infty > \varepsilon \right)$
    for the Gaussian kernel on $[-3, 3]$ with $\sigma = 1$ and $D = 500$,
    based on $1\,000$ evaluations (black),
    numerical integration of the bounds from \cref{thm:uni:tilde,thm:uni:breve}
    (same colors as \cref{fig:bounds}),
    and the bounds of \cref{thm:uni:conc,thm:uni:conc:breve}
    using the empirical mean (yellow).
  }
  \label{fig:survivals}
\end{figure}

The mean of the mean squared error, on the other hand,
exactly follows the expectation of \cref{sec:error:l2}
using $\mu$ as the uniform distribution on $\X^2$:
in this case,
$\E \norm{\tilde f}_\mu \approx 0.66 / D$,
$\E \norm{\breve f}_\mu \approx 0.83 / D$.
(This is natural, as the expectation is exact.)
Convergence to that mean, however,
is substantially faster than guaranteed
by the McDiarmid bound of \cref{thm:l2:tilde,thm:l2:breve}.
We omit the plot due to space constraints.

\subsection{MAXIMUM MEAN DISCREPANCY}
We now turn to the problem of computing the \acro{mmd} with a Fourier embedding.
Specifically, we consider the problem of distinguishing the standard normal distribution
$\mathcal{N}(0, I_p)$
from the two-dimensional mixture
$0.95 \mathcal{N}(0, I_2) + 0.05 \mathcal{N}(0, \tfrac14 I_2)$.
We take fixed sample sets $X$ and $Y$ each of size $1\,000$
and compute the biased \acro{mmd} estimate
with varying $D$ for both $\tilde z$ and $\breve z$,
we used a Gaussian kernel of bandwidth 1.
The mean absolute errors of the resulting estimates are shown in \cref{fig:mmd}.
$\tilde z$ performs mildly better than $\breve z$.

Again, the McDiarmid bound of \cref{sec:downstream:mmd}
predicts that the mean absolute error decays as $O(1 / \sqrt{D})$,
but with too high a multiplicative constant;
the 95\% confidence interval for the exponent of $D$ is
$[-0.515, -0.468]$ for $\tilde z$
and $[-0.520, -0.486]$ for $\breve z$.
We also know that the expected root mean squared error
decays like $O(1 / \sqrt{D})$
via \cref{eq:mmd:variance}.

\begin{figure}[tbh]
  \centering
  \includegraphics[width=.48\textwidth]{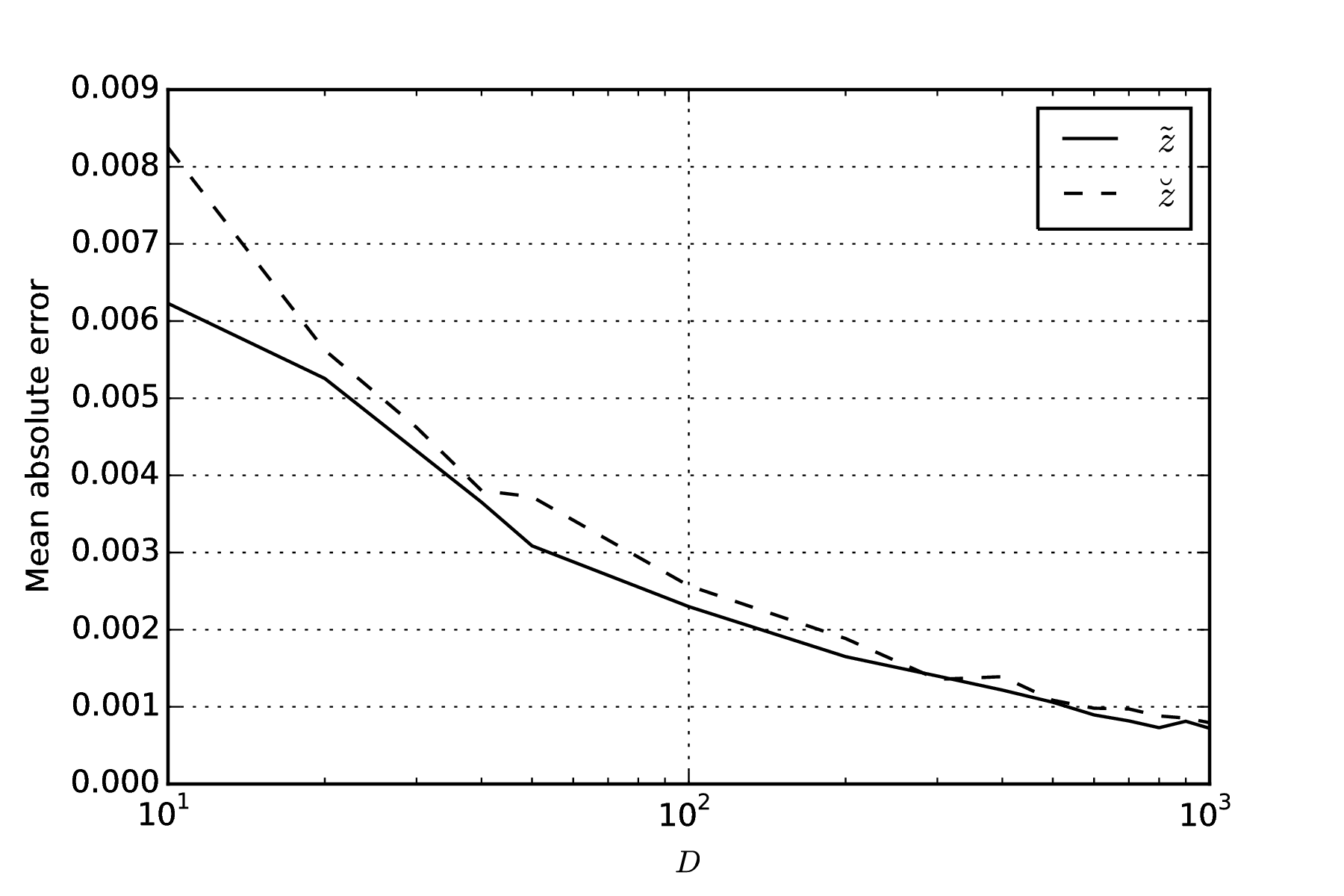}
  \caption{
    Mean absolute error of the biased estimator for $\acro{mmd}(X, Y)$,
    based on 100 evaluations.
  }
  \label{fig:mmd}
\end{figure}

\section{DISCUSSION}
We provide a novel investigation of the approximation error of the popular random Fourier features,
tightening existing bounds and showing new ones,
including an analytic bound on $\E \norm{f}_\infty$ and exponential concentration about its mean,
as well as an exact form for $\E \norm{f}_\mu$ and exponential concentration in that case as well.
We also extend previous results on the change in learned models due to kernel approximation.
We verify some aspects of these bounds empirically for the Gaussian kernel.
We also point out that,
of the two embeddings provided by \textcite{rks},
the $\tilde z$ embedding (with half as many sampled frequencies, but no additional noise due to phase shifts)
is superior in the most common case of the Gaussian kernel.

\subsubsection*{Acknowledgments}
This work was funded in part by DARPA grant FA87501220324.
DJS is also supported by a Sandia Campus Executive Program fellowship.

\clearpage
\subsubsection*{References}
\printbibliography[heading=none]

\ifshortversion\else
\onecolumn
\clearpage
\appendix
\allowdisplaybreaks
\section{PROOFS FOR UNIFORM ERROR BOUND (SECTION \ref{sec:error:uni})} \label{proof:uni}

\subsection{PROOF OF PROPOSITION \ref{thm:uni:tilde}} \label{proof:uni:tilde}

The proof strategy closely follows that of \textcite{rks};
we fill in some (important) details,
tightening some parts of the proof as we go.

Let $\X_\Delta = \{ x - y \mid x, y \in \X \}$.
It's compact, with diameter at most $2 \ell$,
so we can find an $\varepsilon$-net covering $\X_\Delta$
with at most $T = (4 \ell / r)^d$ balls of radius $r$ \parencite[Proposition 5]{cucker:foundations}.
Let $\{\Delta_i\}_{i=1}^T$ denote their centers,
and $L_{\tilde f}$ be the Lipschitz constant of $\tilde f$.
If $\abs{\tilde f(\Delta_i)} < \varepsilon / 2$ for all $i$
and $L_{\tilde f} < \varepsilon / (2 r)$,
then $\abs{\tilde f(\Delta)} < \varepsilon$ for all $\Delta \in \M_\Delta$.

Let $\tilde z_i(x) := \begin{bmatrix} \sin(\omega_i\tp x) & \cos(\omega_i\tp x) \end{bmatrix}\tp$,
so that $\displaystyle \tilde z(x)\tp \tilde z(y) = \frac{1}{D/2} \sum_{i=1}^{D/2} \tilde z_i(x)\tp \tilde z_i(y)$.

\subsubsection{Regularity Condition} \label{sec:regularity-tilde}
We will first need to establish that $\E \nabla \tilde s(\Delta) = \nabla \E \tilde s(\Delta) = \nabla k(\Delta)$.
This can be proved via the following form of the Leibniz rule,
quoted verbatim from \textcite{leibnizrule}:

\begin{theorem*}[{\cite[Theorem 2]{leibnizrule}}]
Let $X$ be an open subset of $\R$,
and $\Omega$ be a measure space.
Suppose $f : X \times \Omega \to \R$
satisfies the following conditions:

\begin{enumerate}
  \item $f(x, \omega)$ is a Lebesgue-integrable function of $\omega$ for each $x \in X$.
  \item For almost all $\omega \in \Omega$, the derivative $\frac{\partial f(x, \omega)}{\partial x}$ exists for all $x \in X$.
  \item There is an integrable function $\Theta : \Omega \to \R$ such that $\Abs{\frac{\partial f(x, \omega)}{\partial x}} \le \Theta(\omega)$ for all $x \in X$.
\end{enumerate}

Then for all $x \in X$,
\[
  \frac{\ud}{\ud x} \int_\Omega f(x, \omega) \,\ud\omega
  = \int_\Omega \frac{\partial}{\partial x} f(x, \omega) \,\ud\omega
.\]
\end{theorem*}

Define the function $\tilde g_{x,y}^i(t, \omega) : \R \times \Omega \to \R$
by $\tilde g_{x,y}^i(t, \omega) = \tilde s_\omega(x + t e_i, y)$,
where $e_i$ is the $i$th standard basis vector,
and $\omega$ is the tuple of all the $\omega_i$ used in $\tilde z$.
$\tilde g_{x,y}^i(t, \cdot)$ is Lebesgue integrable in $\omega$, since
\[ \int \tilde g_{x,y}^i(t, \omega) \,\ud\omega = \E \tilde s(x + t e_i, y) = k(x + t e_i, y) < \infty. \]
For any $\omega \in \Omega$, $\frac{\partial}{\partial t} \tilde g_{x, y}^i(t, \omega)$ exists,
and satisfies:
\begin{align}
  \E_\omega \Abs{\frac{\partial}{\partial t} g_{x,y}^i(t, \omega)}
  &= \E_\omega \Abs{ \frac{2}{D} \sum_{j=1}^{D/2}
      \sin(\omega_j\tp y) \frac{\partial}{\partial t} \sin(\omega_j\tp x + t \omega_{ji})
    + \cos(\omega_j\tp y) \frac{\partial}{\partial t} \cos(\omega_j\tp x + t \omega_{ji})
    }
\\&= \E_\omega \Abs{ \frac{2}{D} \sum_{j=1}^{D/2}
      \omega_{ji} \sin(\omega_j\tp y) \cos(\omega_j\tp x + t \omega_{ji})
    - \omega_{ji} \cos(\omega_j\tp y) \sin(\omega_j\tp x + t \omega_{ji})
    }
\\&\le \E_\omega \left[ \frac{2}{D} \sum_{j=1}^{D/2}
      \Abs{ \omega_{ji} \sin(\omega_j\tp y) \cos(\omega_j\tp x + t \omega_{ji}) }
    + \Abs{ \omega_{ji} \cos(\omega_j\tp y) \sin(\omega_j\tp x + t \omega_{ji}) }
    \right]
\\&\le \E_\omega \left[ \frac{2}{D} \sum_{j=1}^{D/2} 2 \Abs{\omega_{ji}} \right]
\\&\le 2 \E_\omega \Abs{\omega}
,\end{align}
which is finite since the first moment of $\omega$ is assumed to exist.

Thus we have
$\frac{\partial}{\partial x_i} \E \tilde s(x, y) = \E \frac{\partial}{\partial x_i} \tilde s(x, y)$.
The same holds for $y$ by symmetry.
Combining the results for each component,
we get as desired that $\E \nabla_\Delta s(x, y) = \nabla_\Delta \E s(x, y)$.

\subsubsection{Lipschitz Constant} \label{sec:lip-tilde}
Since $\tilde f$ is differentiable,
$L_{\tilde f} = \Norm{\nabla \tilde f(\Delta^*)}$,
where $\Delta^* = \argmax_{\Delta \in \M_\Delta} \Norm{\nabla \tilde f(\Delta)}$.

Via Jensen's inequality,
$\E\Norm{\nabla \tilde s(\Delta)} \ge \Norm{\E \nabla \tilde s(\Delta)}$.
Now, letting $\Delta^* = x^* - y^*$:
\begin{align}
     \E[ L_{\tilde f}^2 ]
  &= \E\left[ \Norm{\nabla \tilde s(\Delta^*) - \nabla k(\Delta^*)}^2 \right]
\\&= \E_{\Delta^*}\!\Bigg[
     \E\left[ \Norm{\nabla \tilde s(\Delta^*)}^2 \right]
   - 2 \Norm{\nabla k(\Delta^*)} \E\big[ \Norm{\nabla \tilde s(\Delta^*)} \big]
   + \Norm{\nabla k(\Delta^*)}^2
   \Bigg]
\\&\le \E_{\Delta^*}\!\Bigg[
     \E\left[ \Norm{\nabla \tilde s(\Delta^*)}^2 \right]
   - 2 \Norm{\nabla k(\Delta^*)}^2
   + \Norm{\nabla k(\Delta^*)}^2
   \Bigg]
\\&= \E\left[ \Norm{\nabla \tilde s(\Delta^*)}^2 \right]
   - \E_{\Delta^*}\!\left[ \Norm{\nabla k(\Delta^*)}^2 \right]
\\&\le \E \Norm{\nabla \tilde s(\Delta^*)}^2
\\&= \E \Norm{ \nabla \tilde{z}(x^*)\tp \tilde{z}(y^*) }^2
\\&= \E \Norm{ \nabla \frac{1}{D/2} \sum_{i=1}^{D/2} \tilde{z}_i(x^*)\tp \tilde{z}_i(y^*) }^2
\\&= \E \Norm{ \nabla \tilde{z}_i(x^*)\tp \tilde{z}_i(y^*) }^2
\label{eq:lip-generic}
\\&= \E \Norm{ \nabla \cos(\omega\tp \Delta^*) }^2
\\&= \E \Norm{ - \sin(\omega\tp \Delta^*) \, \omega }^2
\\&= \E\left[ \sin^2(\omega\tp \Delta^*) \Norm{\omega}^2 \right]
\\&\le \E\left[ \Norm{\omega}^2 \right]
\label{eq:lip-sigma_p}
   = \sigma_p^2
.\end{align}

We can thus use Markov's inequality:
\begin{align}
  \Pr\left( L_{\tilde f} \ge \frac{\varepsilon}{2 r} \right)
  &= \Pr\left( L_{\tilde f}^2 \ge \left( \frac{\varepsilon}{2 r} \right)^2 \right)
   \le \sigma_p^2 \left( \frac{2 r}{\varepsilon} \right)^2
.\end{align}

\subsubsection{Anchor Points}
For any fixed $\Delta = x - y$,
$\tilde f(\Delta)$ is a mean of $D/2$ terms with expectation $k(x, y)$ bounded by $\pm 1$.
Applying Hoeffding's inequality and a union bound:
\[
      \Pr\left( \bigcup_{i=1}^T \abs{\tilde f(\Delta_i)} \ge \tfrac12 \varepsilon \right)
  \le T \Pr\left( \abs{\tilde f(\Delta)} \ge \tfrac12 \varepsilon \right)
  \le 2 T \exp\left( - \frac{2 \frac{D}{2} \left(\frac{\varepsilon}{2}\right)^2}{(1 - (-1))^2} \right)
  = 2 T \exp\left( - \frac{D \varepsilon^2}{16} \right)
.\]
Since we know the variance of each term from \cref{eq:var-s-tilde},
we could alternatively use Bernstein's inequality:
\begin{align}
  T \Pr\left( \Abs{\tilde f(\Delta)} > \tfrac12 \varepsilon \right)
  &\le 2 T \exp\left(
    - \frac{\frac{D}{2} \frac{\varepsilon^2}{4}}%
           {2 \Var[\cos(\omega\tp \Delta)] + \frac23 \varepsilon}
  \right)
   = 2 T \exp\left(
    - \frac{D \varepsilon^2}%
           {16 \left( \Var[\cos(\omega\tp \Delta)] + \frac13 \varepsilon \right)}
  \right)
.\end{align}
This is a better bound when
$\Var[\cos(\omega\tp \Delta)] + \frac13 \varepsilon < 1$;
for the RBF kernel, this is true whenever $\varepsilon < \tfrac32$,
and the improvement is bigger when $k(\Delta)$ is large or $\varepsilon$ is small.

To unify the two,
let $\alpha_\varepsilon := \min\left( 1, \max_{\Delta \in \M_\Delta} \tfrac12 + \tfrac12 k(2 \Delta) - k(\Delta^2) + \tfrac13 \varepsilon \right)$.
Then
\[
      \Pr\left( \bigcup_{i=1}^T \abs{\tilde f(\Delta_i)} \ge \tfrac12 \varepsilon \right)
  \le 2 T \exp\left( - \frac{D \varepsilon^2}{16 \alpha_\varepsilon} \right)
.\]

\subsubsection{Optimizing Over \texorpdfstring{$r$}{r}}
Combining these two bounds, we have a bound in terms of $r$:
\begin{align}
       \Pr\left( \sup_{\Delta \in \M_\Delta} \Abs{\tilde f(\Delta)} \le \varepsilon \right)
  &\ge 1 - \kappa_1 r^{-d} - \kappa_2 r^2
,\end{align}
letting
$\kappa_1 = 2 (4 \ell)^d \exp\left(- \frac{D \varepsilon^2}{16 \alpha_\varepsilon} \right)$,
$\kappa_2 = 4 \sigma_p^2 \varepsilon^{-2}$.

If we choose $r = (\kappa_1 / \kappa_2)^{1 / (d + 2)}$,
as did \textcite{rks},
the bound again becomes
$1 - 2 \kappa_1^{\frac{2}{d + 2}} \kappa_2^{\frac{d}{d + 2}}$.
But we could instead maximize the bound by choosing $r$ such that
$d \kappa_1 r^{-d-1} - 2 \kappa_2 r = 0$,
i.e.\ $r = \left( \frac{d \kappa_1}{2 \kappa_2} \right)^{\frac{1}{d+2}}$.
Then the bound becomes
$
1 - \left( \left( \frac{d}{2} \right)^\frac{-d}{d+2} + \left(\frac{d}{2}\right)^\frac{2}{d+2} \right)
\kappa_1^\frac{2}{d+2} \kappa_2^\frac{d}{d+2}
$:
\begin{align}
       \Pr\left( \sup_{\Delta \in \M_\Delta} \Abs{\tilde f(\Delta)} > \varepsilon \right)
  &\le
    \left( \left( \tfrac{d}{2} \right)^\frac{-d}{d+2} + \left(\tfrac{d}{2}\right)^\frac{2}{d+2} \right)
    \left( 2 (4 \ell)^d \exp\left(-\frac{D \varepsilon^2}{16 \alpha_\varepsilon} \right) \right)^\frac{2}{d + 2}
    \left( 4 \sigma_p^2 \varepsilon^{-2} \right)^\frac{d}{d + 2}
\\&= 
    \left( \left( \tfrac{d}{2} \right)^\frac{-d}{d+2} + \left(\tfrac{d}{2}\right)^\frac{2}{d+2} \right)
    2^\frac{2+4d+2d}{d+2}
    \left( \frac{\sigma_p \ell}{\varepsilon} \right)^\frac{2 d}{d + 2}
    \exp\left( - \frac{D \varepsilon^2}{8 (d + 2) \alpha_\varepsilon} \right)
\label{eq:bound-tilde-tight}
\\&=
    \left( \left( \tfrac{d}{2} \right)^\frac{-d}{d+2} + \left(\tfrac{d}{2}\right)^\frac{2}{d+2} \right)
    2^\frac{6d+2}{d+2}
    \left( \frac{\sigma_p \ell}{\varepsilon} \right)^\frac{2}{1 + 2/d}
    \exp\left( - \frac{D \varepsilon^2}{8 (d + 2) \alpha_\varepsilon} \right)
\label{eq:bound-tilde}
.\end{align}

For $\varepsilon \le \sigma_p \ell$, we can loosen the exponent on the middle term to 2,
though in low dimensions we have a somewhat sharper bound.
We no longer need the $\ell > 1$ assumption of the original proof.

To prove the final statement of \cref{thm:uni:tilde},
simply set \cref{eq:bound-tilde-tight} to be at most $\delta$ and solve for $D$.

\subsection{PROOF OF PROPOSITION \ref{thm:uni:breve}} \label{proof:uni:breve}
We will follow the proof strategy of \cref{thm:uni:tilde} as closely as possible.

Our approximation is now
$\breve s(x, y) = \breve{z}(x)\tp \breve{z}(y)$,
and the error is $\breve f(x, y) = \breve s(x, y) - k(y, x)$.
Note that $\breve s$ and $\breve f$ are not shift-invariant:
for example, with $D = 1$,
$\breve s(x, y) = \cos(\omega\tp\Delta) + \cos(\omega\tp(x+y) + 2b)$
but $\breve s(\Delta, 0) = \cos(\omega\tp\Delta) + \cos(\omega\tp\Delta + 2b)$.

Let $q = \begin{bmatrix}x \\ y\end{bmatrix} \in \X^{2}$
denote the argument to these functions.
$\X^2$ is a compact set in $\R^{2d}$ with diameter $\sqrt{2} \ell$,
so we can cover it with an $\varepsilon$-net
using at most
$T
= \left( 2 \sqrt{2} \ell / r \right)^{2d}$
balls of radius $r$.
Let $\{ q_i \}_{i=1}^T$ denote their centers,
and $L_f$ be the Lipschitz constant of $f : \R^{2d} \to \R$.

\subsubsection{Regularity Condition} \label{sec:regularity-breve}
To show $\E \nabla \breve s(q) = \nabla \E \breve s(q)$,
we can define $\breve g_{x,y}^i(t, \omega)$ analogously to in \cref{sec:regularity-tilde},
where here $\omega$ contains all the $\omega_i$ and $b_i$ variables used in $\breve z$.
We then have:
\begin{align}
  \E_\omega \Abs{\frac{\partial \breve g_{x,y}^i(t, \omega)}{\partial t}}
  &= \E_\omega \Abs{ \frac{1}{D} \sum_{j=1}^{D}
      - \omega_{ji} \cos(\omega_j\tp y + b_j) \sin(\omega_j\tp x + t \omega_{ji} + b_j)
    }
   \le \E_\omega \left[ \frac{1}{D} \sum_{j=1}^{D} \Abs{\omega_{ji}} \right]
   \le \E_\omega \Abs{\omega}
,\end{align}
which we have assumed to be finite.

\subsubsection{Lipschitz Constant} \label{sec:lip-breve}
The argument follows that of \cref{sec:lip-tilde} up to \cref{eq:lip-generic},
using $q^*$ in place of $\Delta^*$.
Then:
\begin{align}
     \E[L_{\breve f}^2]
  &\le \E \Norm{\nabla \breve{s}(q^*)}^2
\\&= \E \Norm{ \nabla_q \left( 2 \cos(\omega\tp x + b) \cos(\omega\tp y + b) \right) }^2
\\&= \E\left[
    \Norm{\nabla_x \left( 2 \cos(\omega\tp x + b) \cos(\omega\tp y + b) \right)}^2
  + \Norm{\nabla_y \left( 2 \cos(\omega\tp x + b) \cos(\omega\tp y + b) \right)}^2
  \right]
\\&= \E \left[
     \Norm{-2 \sin(\omega\tp x^* + b) \cos(\omega\tp y^* + b) \,\omega}^2
   + \Norm{-2 \cos(\omega\tp x^* + b) \sin(\omega\tp y^* + b) \,\omega}^2
   \right]
\\&= \E\left[
    4 \left(
      \sin^2(\omega\tp x^* + b) \cos^2(\omega\tp y^* + b)
    + \cos^2(\omega\tp x^* + b) \sin^2(\omega\tp y^* + b)
    \right)
    \Norm{\omega}^2
  \right]
\\&= \E_\omega\left[
    \E_b\left[ 2 - \cos(2 \omega\tp (x^* - y^*)) - \cos(2 \omega\tp (x^* + y^*) + 4 b) \right]
    \Norm{\omega}^2
  \right]
\\&= \E_\omega\left[
    \left( 2 - \cos(2 \omega\tp (x^* - y^*)) \right)
    \Norm{\omega}^2
  \right]
\\&\le 3 \E\Norm\omega^2 = 3 \sigma_p^2
\label{eq:lip-breve-last}
.\end{align}

Following through with Markov's inequality:
\[
\Pr\left( L_{\breve f} \ge \varepsilon / (2 r) \right)
 \le 3 \sigma_p^2 (2 r / \varepsilon)^2
 = 12 (\sigma_p r / \varepsilon)^2
.\]

\subsubsection{Anchor Points}
For any fixed $x, y$,
$\breve s$ takes a mean of $D$ terms with expectation $k(x, y)$ bounded by $\pm 2$.
Using Hoeffding's inequality:
\[
      \Pr\left( \bigcup_{i=1}^T \abs{\breve f(q_i)} \ge \tfrac12 \varepsilon \right)
  \le T \Pr\left( \abs{\breve f(q)} \ge \tfrac12 \varepsilon \right)
  \le 2 T \exp\left( - \frac{2 D \left(\frac{\varepsilon}{2}\right)^2}{(2 - (-2))^2} \right)
  = 2 T \exp\left( - \frac{D \varepsilon^2}{32} \right)
.\]
Since the variance of each term is given by \cref{eq:var-s-breve},
we can instead use Bernstein's inequality:
\begin{align}
  T \Pr\left( \Abs{\breve f(\Delta)} > \tfrac12 \varepsilon \right)
  &\le 2 T \exp\left(
    - \frac{D \frac{\varepsilon^2}{4}}%
           {2 \left( \Var[\cos(\omega\tp \Delta)] + \frac12 \right) + \frac43 \varepsilon}
  \right)
   = 2 T \exp\left(
    - \frac{D \varepsilon^2}%
           {4 + 8 \Var[\cos(\omega\tp \Delta)] + \frac{16}{3} \varepsilon}
  \right)
.\end{align}
Thus Bernstein's gives us a tighter bound if
\begin{gather}
   4 + 8 \Var[\cos(\omega\tp \Delta)] + \frac{16}{3} \varepsilon < 32
\qquad\text{i.e.}\qquad
   2 \Var[\cos(\omega\tp \Delta)] + \frac43 \varepsilon < 7
.\end{gather}

To unify the bounds, define
$\alpha'_\varepsilon = \min\left(1, \max_{\Delta} \tfrac18 + \tfrac14 \Var[\cos(\omega\tp \Delta)] + \frac16 \varepsilon\right)$;
then
\[
      \Pr\left( \bigcup_{i=1}^T \abs{\breve f(q_i)} \ge \tfrac12 \varepsilon \right)
  \le 2 T \exp\left( - \frac{D \varepsilon^2}{32 \alpha'_\varepsilon} \right)
.\]

\subsubsection{Optimizing Over \texorpdfstring{$r$}{}}
Our bound is now of the form
\begin{align}
  \Pr\left( \sup_{q \in \M^2} \Abs{\breve f(q)} \le \varepsilon \right)
  \ge 1 - \kappa_1 r^{-2d} - \kappa_2 r^2
,\end{align}
with $\kappa_1 = 2 \left( 2 \sqrt{2} \ell \right)^{2d} \exp\left( - \frac{D\varepsilon^2}{32 \alpha'_\varepsilon} \right)$
and $\kappa_2 = 12 \sigma_p^2 \varepsilon^{-2}$.

This is maximized by $r$ when
$2 d \kappa_1 r^{-2d-1} - 2 \kappa_2 r = 0$,
i.e.\ $r = \left( \frac{d \kappa_1}{\kappa_2} \right)^{\frac{1}{2d+2}}$.
Substituting that value of $r$ into the bound yields
$
1 - \left( d^\frac{-d}{d+1} + d^\frac{1}{d+1} \right) \kappa_1^\frac{1}{d+1} \kappa_2^\frac{d}{d+1}
$, and thus:
\begin{align}
     \Pr\left( \sup_{q \in \M^2} \Abs{\breve f(q)} > \varepsilon \right)
  &\le
    \left( d^\frac{-d}{d+1} + d^\frac{1}{d+1} \right)
    \left( 2 \left( 2 \sqrt{2} \ell \right)^{2d} \exp\left( - \frac{D\varepsilon^2}{32 \alpha'_\varepsilon} \right) \right)^\frac{1}{d+1}
    \left( 12 \sigma_p^2 \varepsilon^{-2} \right)^\frac{d}{d+1}
\\&=
    \left( d^\frac{-d}{d+1} + d^\frac{1}{d+1} \right)
    2^\frac{1 + 2d + d + 2d}{d+1}
    3^\frac{d}{d+1}
    \left( \frac{\sigma_p \ell}{\varepsilon} \right)^\frac{2d}{d+1}
    \exp\left( - \frac{D \varepsilon^2}{32 (d + 1) \alpha'_\varepsilon} \right)
\label{eq:bound-breve-tight}
\\&=
    \left( d^\frac{-d}{d+1} + d^\frac{1}{d+1} \right)
    2^\frac{5d + 1}{d+1}
    3^\frac{d}{d+1}
    \left( \frac{\sigma_p \ell}{\varepsilon} \right)^\frac{2}{1+1/d}
    \exp\left( - \frac{D \varepsilon^2}{32 (d + 1) \alpha'_\varepsilon} \right)
\label{eq:bound-breve}
.\end{align}

As before, when $\varepsilon \le \sigma_p \ell$
we can loosen the exponent on the middle term to 2;
it is slightly worse than the corresponding exponent of \cref{eq:bound-tilde}
for small $d$.

To prove the final statement of \cref{thm:uni:breve},
set \cref{eq:bound-breve-tight} to be at most $\delta$ and solve for $D$.

\subsection{PROOF OF PROPOSITION \ref{thm:exp-sup-tilde}} \label{proof:exp:tilde}
Consider the $\tilde z$ features,
and recall that we supposed $k$ is $L$-Lipschitz over
$\X_\Delta := \{ x - y \mid x, y \in \X \}$.

Our primary tool will be the following slight generalization of Dudley's entropy integral,
which is a special case of Lemma~13.1 of \textcite{boucheron}.
(The only difference from their Corollary~13.2 is that we maintain the variance factor $v$.)
\begin{theorem*}[{\cite{boucheron}}]
  Let $\T$ be a finite pseudometric space
  and let $(X_t)_{t \in \T}$ be a collection of random variables
  such that for some constant $v > 0$,
  \[
    \log \E e^{\lambda (X_t - X_{t'})}
    \le \frac{1}{2} v \lambda^2 d^2(t, t')
  \]
  for all $t, t' \in \T$ and all $\lambda > 0$.
  Let $\delta = \sup_{t \in \T} d(t, t_0)$.
  Then, for any $t_0 \in \T$,
  \[
    \E\left[ \sup_{t \in \T} X_t - X_{t_0} \right]
    \le 12 \sqrt{v} \int_0^{\delta/2} \sqrt{H(u, \T)} \,\ud{u}
  .\]
\end{theorem*}
Note that, although stated for finite pseudometric spaces,
the result is extensible to seperable pseudometric spaces (such as $\X_\Delta$) by standard arguments.
Here $H(\delta, \T) = \log N(\delta, \T)$, where $N$ is the $\delta$-packing number,
is known as the $\delta$-entropy number.
It is the case that the $\delta$-packing number
is at most the $\frac{\delta}{2}$-covering number,
which Proposition~5 of \textcite{cucker:foundations} bounds.
Thus,
picking $\Delta_0 = 0$ gives $\delta = \ell$,
$H(\delta, \X_\Delta) \le d \log\left( {8 \ell}/{\delta} \right)$,
and
\begin{align}
    \int_0^{\ell / 2}\! \sqrt{H(u, \X_\Delta)} \,\ud u
  &\le \int_0^{\ell / 2}\! \sqrt{d \log(8 \ell / u)} \,\ud u
   = \gamma \ell \sqrt{d}
,\end{align}
where $\gamma := 4 \sqrt{\pi} \erfc(2 \sqrt{\log 2}) + \sqrt{\log 2} \approx 0.964$.

Now,
$
  \frac{2}{D} \left( \cos(\omega_i\tp \Delta) - k(\Delta) - \cos(\omega_i\tp\Delta') + k(\Delta') \right)
$
has mean zero,
and absolute value bounded by
\begin{align}
  \Abs{ \frac{2}{D} \left( \cos(\omega_i\tp \Delta) - k(\Delta) - \cos(\omega_i\tp\Delta') + k(\Delta') \right)}
  &\le \frac{2}{D} \left(
      \Abs{\cos(\omega_i\tp\Delta) - \cos(\omega_i\tp\Delta')}
    + \Abs{k(\Delta) - k(\Delta')}
  \right)
\\&\le \frac2D \left(
      \Abs{\omega_i\tp\Delta - \omega_i\tp\Delta'}
    + L \Norm{\Delta - \Delta'}
  \right)
\\&\le \frac2D \left( \norm{\omega_i} + L \right) \norm{\Delta - \Delta'}
\label{proof:exp:tilde:absbound}
.\end{align}
Thus,
via Hoeffding's lemma \parencite[Lemma~2.2]{boucheron},
each such term has log moment generating function at most
$\frac{2}{D^2} (\norm{\omega_i} + L)^2 \lambda^2 \norm{\Delta - \Delta'}^2$.

This is almost in the form required by Dudley's entropy integral,
except that $\omega_i$ is a random variable.
Thus, for any $r > 0$, define the random process $\tilde g_r$
which is distributed as $\tilde f$
except we require that $\norm{\omega_1} = r$
and $\norm{\omega_i} \le r$ for all $i > 1$.
Since log mgfs of independent variables are additive, we thus have
\begin{align}
  \log \E e^{ \lambda (\tilde g_r(\Delta) - \tilde g_r(\Delta'))}
  &\le
  \frac{1}{D}
  \left( \frac{2}{D} \sum_{i=1}^{D/2} (\norm{\omega_i} + L)^2 \right)
  \lambda^2
  \norm{\Delta - \Delta'}^2
   \le
  \frac{1}{D}
  (r + L)^2
  \lambda^2
  \norm{\Delta - \Delta'}^2
.\end{align}

$\tilde g_r$ satisfies the conditions of the theorem with $v = \frac1D (r + L)^2$.
Now, $\tilde g_r(0) = 0$,
so we have
\[
  \E\left[ \sup_{\Delta \in \X_\Delta} \tilde g_r(\Delta) \right]
  \le \frac{12 \gamma \sqrt{d} \ell}{\sqrt{D}} (r + L)
.\]

But the distribution of $\tilde f$ conditioned on the event $\max_i \norm{\omega_i} = r$
is the same as the distribution of $\tilde g_r$.
Thus
\[
  \E \sup \tilde f
  = \E_r \left[ \E[ \sup \tilde g_r ] \right]
  \le \E_r \left[ \frac{12 \gamma \sqrt{d} \ell}{\sqrt{D}} (r + L) \right]
  = \frac{12 \gamma \sqrt{d} \ell}{\sqrt{D}} (R + L)
\]
where $R := \E \max_{i = 1}^{D / 2} \norm{\omega_i}$.

The same holds for $\E \sup( - \tilde f )$.
Since we have
$\sup \tilde f \ge 0$,
$\sup( - \tilde f ) \ge 0$,
the claim follows from
$\E \left[ \max( \sup \tilde f, \sup(- \tilde f)) \right]
\le \E\left[ \sup \tilde f + \sup( -\tilde f) \right]$.

\subsection{PROOF OF PROPOSITION \ref{thm:exp-sup-breve}} \label{proof:exp:breve}

For the $\breve z$ features,
the error process again must be defined over $\X^2$ due to the non-shift invariant noise.
We still assume that $k$ is $L$-Lipschitz over $\X_\Delta$, however.

Compared to the argument of \cref{proof:exp:tilde},
we have $H(u, \X^2) \le 2 d \log\left( 4 \sqrt{2} \ell / u \right)$.
Unlike $\X_\Delta$, however,
$\X^2$ does not necessarily contain an obvious point $q_0$ to minimize
$\sup_{q \in \X^2} d(q, q_0)$,
nor an obvious minimal value.
We rather consider the ``radius''
$\rho := \sup_{x \in \X} d(x, x_0)$,
achieved by any convenient point $x_0$;
then $\sup_{q \in \X^2} d\left(q, (x_0, x_0) \right) = \sqrt{2} \rho$.
Note that $\tfrac12 \ell \le \rho \le \ell$,
where the lower bound is achieved by $\X$ a ball,
and the upper bound by $\X$ a sphere.
The integral in the bound is then
\begin{align}
  \int_0^{\rho/\sqrt{2}} \sqrt{H(u, \X^2)}
  &\le \int_0^{\rho / \sqrt 2} \sqrt{2 d \log( 4 \sqrt{2} \ell / u )}
\\&= 4 \sqrt{\pi d} \ell
     \erfc\left( \sqrt{ \tfrac12 \log 2 + \log 4 \sqrt{2} \tfrac{\ell}{\rho} } \right)
   + \rho \sqrt{d} \sqrt{ \tfrac52 \log 2 + \log \tfrac{\ell}{\rho}}
\\&=
    \left(
      4 \sqrt{\pi} \erfc\left( \sqrt{ \tfrac12 \log 2 + \log 4 \sqrt{2} \tfrac{\ell}{\rho} } \right)
    + \tfrac{\rho}{\ell} \sqrt{ \tfrac52 \log 2 + \log \tfrac{\ell}{\rho}}
    \right)
    \ell \sqrt{d}
  \label{proof:exp:breve:entropy-rate}
.\end{align}
Calling the term in parentheses $\gamma'_{\ell / \rho}$,
we have that $\gamma'_1 \approx 1.541$, $\gamma'_2 \approx 0.803$,
and it decreases monotonically in between, as shown in \cref{fig:gamma-prime}.

\begin{figure}[hbt]
  \centering
  \includegraphics[width=.45\textwidth]{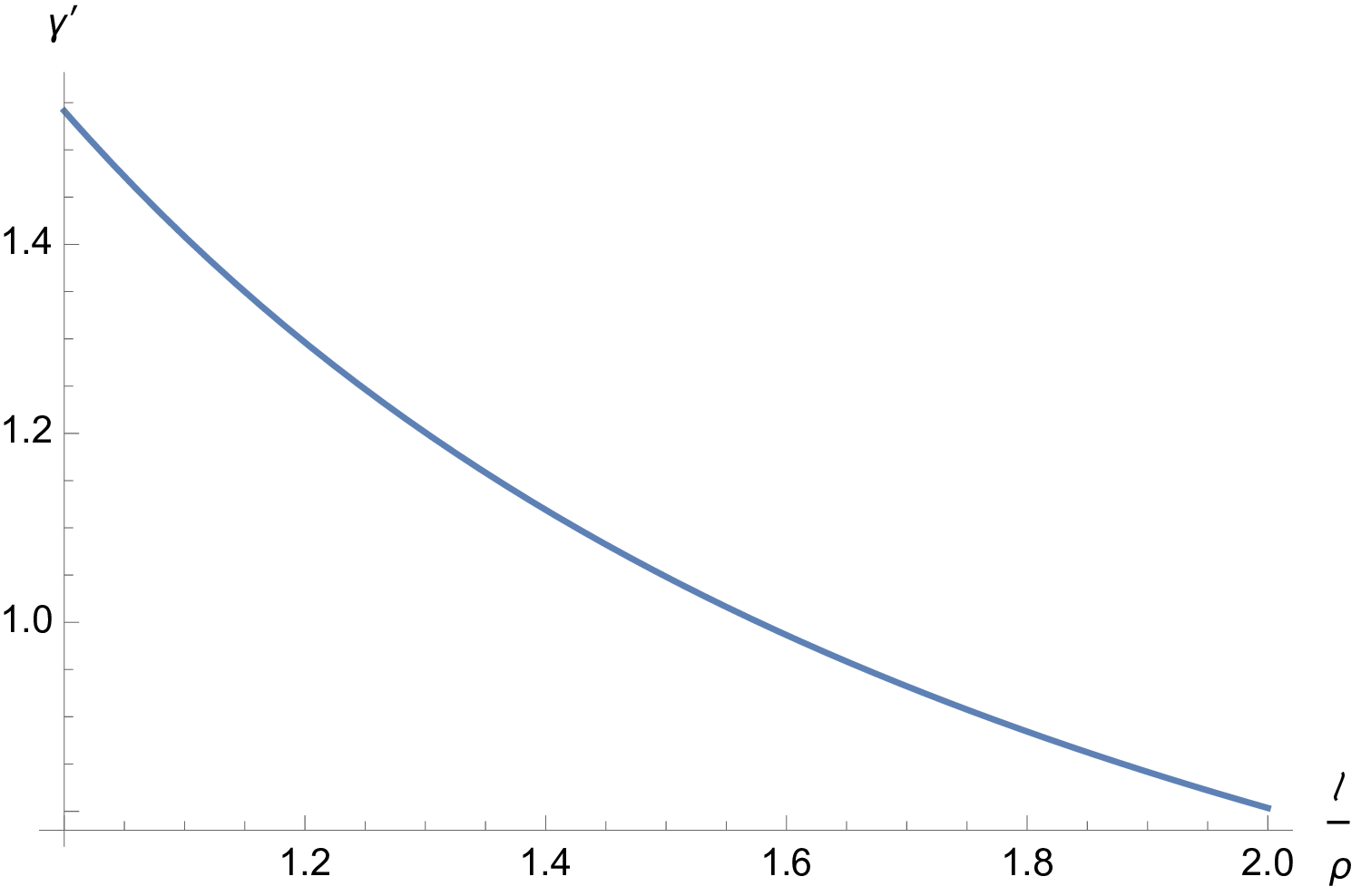}
  \caption{
    The coefficient of \cref{proof:exp:breve:entropy-rate}
    as a function of $\ell / \rho$.
  }
  \label{fig:gamma-prime}
\end{figure}

We will again use the notation of $q = (x, y) \in \X^2$, $\Delta = x - y$, $t = x + y$.
Each term in the sum of $\breve f(q) - \breve f(q')$ has mean zero and absolute value at most
\begin{align}
  \frac1D \lvert
      \cos(\omega_i\tp \Delta) &+ \cos(\omega_i\tp t + 2 b_i) - k(\Delta)
    - \cos(\omega_i\tp \Delta') + \cos(\omega_i\tp t' + 2 b_i) + k(\Delta')
  \rvert
\\&\le \frac1D \left(
      \Abs{\cos(\omega_i\tp \Delta) - \cos(\omega_i\tp \Delta')}
    + \Abs{\cos(\omega_i\tp t + 2 b_i) - \cos(\omega_i\tp t' + 2 b_i)}
    + \Abs{k(\Delta) - k(\Delta')}
  \right)
\\&\le \frac1D \left(
      \norm{\omega_i} \norm{\Delta - \Delta'}
    + \norm{\omega_i} \norm{t - t'}
    + L \norm{\Delta - \Delta'}
  \right)
.\end{align}
Now, in order to cast this in terms of distance on $\X^2$,
let $\delta_x = x - x'$, $\delta_y = y - y'$.
Then
\begin{align}
  \norm{q - q'}^2
  &= \norm{\delta_x}^2 + \norm{\delta_y}^2
  \\
  \left( \norm{\Delta - \Delta'} + \norm{t - t'} \right)^2
  &= \left(
    \sqrt{\norm{\delta_x}^2 + \norm{\delta_y}^2 - 2 \delta_x\tp\delta_y}
  + \sqrt{\norm{\delta_x}^2 + \norm{\delta_y}^2 + 2 \delta_x\tp\delta_y}
  \right)^2
\\&= 2 \norm{\delta_x}^2 + 2 \norm{\delta_y}^2
   + 2 \sqrt{
      \left( \norm{\delta_x}^2 + \norm{\delta_y}^2 \right)^2
      - 4 (\delta_x\tp\delta_y)^2
    }
\\&\le 4 \left( \norm{\delta_x}^2 + \norm{\delta_y}^2 \right)
\\
  \norm{\Delta - \Delta'} + \norm{t - t'}
  &\le 2 \norm{q - q'}
\\
  \norm{\Delta - \Delta'}
  &\le 2 \norm{q - q'}
\end{align}
and so each term in the sum of $\breve f(q) - \breve f(q')$ has absolute value at most
$
  \frac2D \left( \norm{\omega_i} + L \right) \norm{q - q'}
$.
Note that this agrees exactly with \cref{proof:exp:tilde:absbound},
but the sum in $\breve f(q) - \breve f(q')$ has $D$ terms rather than $D / 2$.
Defining $\breve g_r$ analogously to $\tilde g_r$,
we thus get that
\[
  \log \E e^{\lambda (\breve g_r(q) - \breve g_r(q'))}
  \le \frac{2}{D} \left( \frac{1}{D} \sum_{i=1}^{D} \left( \norm{\omega_i} + L \right)^2 \right) \lambda^2 \norm{q - q'}^2
  \le \frac{2}{D} (r + L)^2 \lambda^2 \norm{q - q'}^2
,\]
and the conditions of the theorem hold with $v = \frac{4}{D} (r + L)^2$.
Note that $\E \breve g_r(q_0) = 0$.
Carrying out the rest of the argument, we get that
\[
  \E \sup \breve f
  = \E_r[ \E[ \sup \breve g_r] ]
  \le \E_r\left[ \frac{24 \beta_{\ell / \rho} \ell \sqrt{d}}{\sqrt{D}} (r + L) \right]
  = \frac{24 \beta_{\ell / \rho} \ell \sqrt{d}}{\sqrt{D}} (R + L)
,\]
and similarly for $\E \sup \breve f$.
We do not have a guarantee that $\breve f(q)$ does not have a consistent sign,
and so our bound becomes
\begin{align}
  \E \norm{\breve f}_\infty
  &\le \E\left[ \norm{\breve f}_\infty \mid \breve f \text{ crosses } 0 \right]
      \Pr\left( \breve f \text{ crosses } 0 \right)
    + 3 \Pr\left( \breve f \text{ does not cross } 0 \right)
\\&\le \frac{48 \beta_{\ell / \rho} \ell \sqrt{d}}{\sqrt{D}} (R + L)
      \Pr\left( \breve f \text{ crosses } 0 \right)
    + 3 \Pr\left( \breve f \text{ does not cross } 0 \right)
.\end{align}
$\Pr\left( \breve f \text{ crosses } 0 \right)$
is extremely close to 1 in ``usual'' situations.

\section{PROOFS FOR \texorpdfstring{$L_2$}{L2} ERROR BOUND (SECTION \ref{sec:error:l2})} \label{proof:l2}

\subsection{BOUNDED CHANGE IN \texorpdfstring{$\norm{\tilde f}_\mu^2$}{L2 NORM OF SHIFT-INVARIANT ERROR}} \label{proof:l2:diff:tilde}

Viewing $\norm{\tilde f}_\mu^2$ as a function of $\omega_1, \dots, \omega_{D/2}$,
we can bound the change due to replacing $\omega_1$ by $\hat \omega_1$.
(By symmetry, the change is the same for any $\omega_i$.)
\begin{align}
       \Big\lvert
         \norm{\tilde f}_\mu^2&(\omega_1, \omega_2, \dots, \omega_{D/2})
       - \norm{\tilde f}_\mu^2(\hat\omega_1, \omega_2, \dots, \omega_{D/2})
       \Big\rvert
\\& =  \left\lvert
         \int_{\X^2} \left(
             \frac{2}{D} \cos(\omega_1\tp (x - y))
           + \frac{2}{D} \sum_{i=2}^{D/2} \cos(\omega_i\tp (x - y))
           - k(x, y)
           \right)^2 \ud\mu(x, y)
\right.\\&\quad\left.
       - \int_{\X^2} \left(
             \frac{2}{D} \cos(\hat\omega_1\tp (x - y))
           + \frac{2}{D} \sum_{i=2}^{D/2} \cos(\omega_i\tp (x - y))
           - k(x, y)
           \right)^2 \ud\mu(x, y)
       \right\rvert
\\& =  \left\lvert
         \frac{4}{D^2} \int_{\X^2} \cos^2(\omega_1\tp (x - y)) \ud\mu(x, y)
       + \int_{\X^2} \left(
             \frac{2}{D} \sum_{i=2}^{D/2} \cos(\omega_i\tp (x - y))
           - k(x, y)
           \right)^2 \ud\mu(x, y)
\right.\\&\quad\left.
       + 2 \int_{\X^2}
           \frac{2}{D} \cos(\omega_1\tp (x - y))
           \left(
             \frac{2}{D} \sum_{i=2}^{D/2} \cos(\omega_i\tp (x - y))
           - k(x, y)
           \right)
           \ud\mu(x, y)
\right.\\&\quad\left.
       - \frac{4}{D^2} \int_{\X^2} \cos^2(\hat\omega_1\tp (x - y)) \ud\mu(x, y)
       - \int_{\X^2} \left(
             \frac{2}{D} \sum_{i=2}^{D/2} \cos(\omega_i\tp (x - y))
           - k(x, y)
           \right)^2 \ud\mu(x, y)
\right.\\&\quad\left.
       - 2 \int_{\X^2}
           \frac{2}{D} \cos(\hat\omega_1\tp (x - y))
           \left(
             \frac{2}{D} \sum_{i=2}^{D/2} \cos(\omega_i\tp (x - y))
           - k(x, y)
           \right)
           \ud\mu(x, y)
       \right\rvert
\\& =  \left\lvert
         \frac{4}{D^2} \int_{\X^2} \left( \cos^2(\omega_1\tp (x - y)) - \cos^2(\hat\omega_1\tp (x - y)) \right) \ud\mu(x, y)
\right.\\&\quad\left.
       + \frac{4}{D} \int_{\X^2}
           \left( \cos(\omega_1\tp (x - y)) - \cos(\hat\omega_1\tp (x - y)) \right)
           \left(
               \frac{2}{D} \sum_{i=2}^{D/2} \cos(\omega_i\tp (x - y))
             - k(x, y)
           \right)
           \ud\mu(x, y)
       \right\rvert
\\&\le \frac{4}{D^2} \int_{\X^2} \ud\mu(x, y)
     + \frac{4}{D}   \int_{\X^2} 4 \ud\mu(x, y)
\\& =  \left( \frac{4}{D^2} + \frac{16}{D} \right) \mu(\X^2)
    =  \frac{16 D + 4}{D^2} \mu(\X^2)
.\end{align}

\subsection{BOUNDED CHANGE IN \texorpdfstring{$\norm{\breve f}_\mu^2$}{L2 NORM OF PHASE-SHIFT ERROR}} \label{proof:l2:diff:breve}
We can do essentially the same thing for $\breve f$:
\begin{align}
       \Big\lvert
         \norm{\breve f}_\mu^2&(\omega_1, \omega_2, \dots, \omega_{D/2})
       - \norm{\breve f}_\mu^2(\hat\omega_1, \omega_2, \dots, \omega_{D/2})
       \Big\rvert
\\& =  \left\lvert
         \int_{\X^2} \left(
             \frac{2}{D} \left( \cos(\omega_1\tp (x - y)) + \cos(\omega_1\tp (x + y) + 2 b_i) \right)
\right.\right.\\&\qquad\qquad\left.\left.
           + \frac{2}{D} \sum_{i=2}^{D/2} \left[ \cos(\omega_i\tp (x - y)) + \cos(\omega_i\tp (x + y) + 2 b_i) \right]
           - k(x, y)
           \right)^2 \ud\mu(x, y)
\right.\\&\quad\left.
       - \int_{\X^2} \left(
             \frac{2}{D} \left( \cos(\hat\omega_1\tp (x - y)) + \cos(\hat\omega_1\tp (x + y) + 2 b_i) \right)
\right.\right.\\&\qquad\qquad\left.\left.
           + \frac{2}{D} \sum_{i=2}^{D/2} \left[ \cos(\omega_i\tp (x - y)) + \cos(\omega_i\tp (x - y) + 2 b_i) \right]
           - k(x, y)
           \right)^2 \ud\mu(x, y)
       \right\rvert
\\& =  \left\lvert
         \frac{4}{D^2} \int_{\X^2}
           \left(
             \left( \cos(\omega_1\tp (x - y)) + \cos(\omega_1\tp(x + y) + 2 b_i) \right)^2
           - \left( \cos(\hat\omega_1\tp (x - y)) + \cos(\hat\omega_1\tp(x + y) + 2 b_i) \right)^2
           \right)
           \ud\mu(x, y)
\right.\\&\quad\left.
       + \frac{4}{D} \int_{\X^2}
           \left(
             \cos(\omega_1\tp (x - y)) + \cos(\omega_1\tp(x + y) + 2 b_i)
           - \cos(\hat\omega_1\tp (x - y)) - \cos(\hat\omega_1\tp(x + y) + 2 b_i)
           \right)
\right.\\&\qquad\qquad\qquad\left.
           \left(
             \frac{2}{D} \sum_{i=2}^{D/2} \left[ \cos(\omega_i\tp (x - y)) + \cos(\omega_i\tp (x + y) + 2 b_i) \right]
           - k(x, y)
           \right)
           \ud\mu(x, y)
       \right\rvert
\\&\le \frac{4}{D^2} \int_{\X^2} 8 \,\ud\mu(x, y)
     + \frac{4}{D} \int_{\X^2} 8 \,\ud\mu(x, y)
\\& =  \frac{32}{D^2} \mu(\X^2) + \frac{32}{D} \mu(\X^2)
\\& =  32 \frac{D + 1}{D^2} \mu(\X^2)
.\end{align}
\fi

\end{document}